\documentclass{article}

\usepackage{graphicx, amsmath,amsthm, amssymb}

\usepackage{microtype}
\usepackage{graphicx}
\usepackage{subfigure}
\usepackage{booktabs} 

\usepackage{hyperref}

\usepackage[accepted]{icml2020}

\def\O{\mathcal{O}}

\def\0{\mathbf{0}}
\def\1{\mathbf{1}}

\def\x{\mathbf{x}}
\def\\x{\overline{\mathbf{x}}}
\def\y{\mathbf{y}}
\def\\y{\overline{\mathbf{y}}}

\def\mN{\mathcal N}
\def\mP{\mathcal P}
\def\mR{\mathcal R}

\newtheorem{thm}{Theorem}
\newtheorem{lem}{Lemma}

\newtheorem{defi}{Definition}
\newtheorem{rem}{Remark}

\newtheorem{ex}{Example}

\icmltitlerunning{On the Number of Linear Regions of Convolutional Neural Networks}

\begin{document}

\twocolumn[
\icmltitle{On the Number of Linear Regions of Convolutional Neural Networks}




\begin{icmlauthorlist}
\icmlauthor{Huan Xiong}{mbzuai}
\icmlauthor{Lei Huang}{iiai}
\icmlauthor{Mengyang Yu}{iiai}
\icmlauthor{Li Liu}{iiai}
\icmlauthor{Fan Zhu}{iiai}
\icmlauthor{Ling Shao}{mbzuai,iiai}
\end{icmlauthorlist}

\icmlaffiliation{iiai}{Inception Institute of Artificial Intelligence,
  Abu Dhabi, UAE}
\icmlaffiliation{mbzuai}{Mohamed bin Zayed University of Artificial Intelligence, UAE}

\icmlcorrespondingauthor{Huan Xiong}{huan.xiong@mbzuai.ac.ae}

\icmlkeywords{Machine Learning, ICML}

\vskip 0.3in
]



\printAffiliationsAndNotice{} 

\begin{abstract}
One fundamental problem in deep learning is understanding the outstanding performance of deep Neural Networks (NNs) in practice. One explanation for the superiority of NNs is that they can realize a large class of complicated functions, i.e., they have powerful expressivity. The expressivity of a ReLU NN can be quantified by the maximal number of linear regions it can separate its input space into. In this paper, we provide several mathematical results needed for studying the linear regions of CNNs, and use them to derive the maximal and average numbers of linear regions for one-layer ReLU CNNs. Furthermore, we obtain upper and lower bounds for the number of linear regions of multi-layer ReLU CNNs. Our results suggest that deeper CNNs have more powerful expressivity than their shallow counterparts, while CNNs have more expressivity than fully-connected NNs per parameter. 
\end{abstract}

\section{Introduction}
Over the past decade, deep Neural Networks (NNs), especially deep Convolutional Neural Networks (CNNs), have attracted much attention and achieved state-of-the-art results in many machine learning tasks, such as speech recognition, image classification, and video games \cite{hinton2012deep,goodfellow2013maxout,sainath2013deep,abdel2014convolutional,silver2016mastering}. Various popular and powerful CNNs, such as AlexNet \cite{krizhevsky2012imagenet}, VGGNet \cite{simonyan2014very}, GoogleNet \cite{szegedy2015going} and ResNet \cite{he2016deep}, have  empirically shown that applying deeper networks can significantly improve the performance of various network architectures. 
A key problem in the study of deep learning is to understand why neural networks, especially very deep neural networks, perform well in practice.

One explanation for the superiority of NNs is their powerful expressivity, i.e., they can represent a large classes of functions arisen in practice. It has been shown that an NN with only one hidden layer can adequately approximate any given continuous function if its width is large enough \cite{  cybenko1989approximation, funahashi1989approximate, hornik1991approximation,barron1994approximation}. However, normally the width of such a hidden layer has to be exponentially large in order to approximate a given function to arbitrary precision. In contrast, if multiple layers are involved, \cite{hanin2017universal,hanin2017approximating,lu2017expressive} proved that to approximate any given Lebesgue-integrable function from $\mathbb{R}^n$ to $\mathbb{R}$ to arbitrary precision, one only needs to apply some multi-layer NN with width at most $n+1$, while the depth depends on the given function and may be very large. Although these approximate results show that NNs can represent a large class of functions, there are few hints on how to determine the suitable architectures needed to realise a given function or which architectures are more efficient. Recently, several theoretical studies have been conducted to compare the efficiency of distinct architectures. It was proved in \cite{telgarsky2015representation,telgarsky2016benefits,arora2016understanding} that certain functions realized by some deep architectures will require a shallow network with exponentially more parameters to represent. For example,  \cite{telgarsky2016benefits} showed that, for any positive integer $n$, there exist some networks with depth $\Theta(n^3)$, width $\Theta(1)$, and $\Theta(1)$ parameters, that cannot be approximated by an $\O(n)$-layer network unless it has a width of $\Omega(2^n)$. Their results reveal that deeper networks usually have more powerful expressivity of functions, which provides an explanation for why deeper networks outperform shallow networks with the same number of parameters in many practical tasks. 

A natural measure for characterizing the expressivity of NNs is the maximal number of distinct linear regions \cite{pascanu2013number} in the domain of functions that can be computed by NNs.
Among this direction, people mainly focus on NNs whose activation functions are Rectified Linear Units (ReLUs), which were first introduced in $2000$ \cite{hahnloser2000digital, hahnloser2001permitted} and have been widely adopted in various architectures since $2011$ \cite{glorot2011deep}. It is known that the composition of piecewise linear\footnote{Although ``piecewise affine" would be more accurate, we use ``piecewise linear" here since it is a conventional concept.} functions  is still piecewise linear; thus, every feed-forward ReLU NN (neural network with only ReLU activations and linear hidden layers) with certain parameters can be seen as a piecewise linear function. This means that the input space of a ReLU network can be divided into several distinct pieces (we call them linear regions), such that the function represented by the network is affine when restricted to each piece. Then, the expressivity of a ReLU network can be quantified by the maximal number of linear regions it can separate its input space into. Pascanu et al. \yrcite{pascanu2013number} first considered a one-layer fully-connected ReLU network with $n_0$ inputs and $n_1$ hidden neurons, and showed that its maximal number of linear regions equals $\sum_{i=0}^{n_0} \binom{n_1}{i}$ by translating this problem to a counting problem of regions of hyperplane arrangements in general position (the definition of ``general position" is given in the Section $1$ of the Supplementary Material), then directly applying Zaslavsky's Theorem \cite{zaslavsky1975facing,Stanley04anintroduction}.   Furthermore, using the idea of identifying distinct linear regions, they derived a lower bound $\left(\prod_{l=0}^{L-1}\left\lfloor \frac{n_l}{n_0} \right\rfloor\right)\sum_{i=0}^{n_0} \binom{n_L}{i}$ for the maximal number of linear regions of a fully-connected ReLU network with $n_0$ inputs and $L$ hidden layers of widths $n_1 , n_2 , \ldots , n_L$. Based on these results, they concluded that deep  fully-connected ReLU NNs have exponentially more maximal linear regions than their shallow counterparts with the same number of parameters. Later, the lower bound was improved to $\left(\prod_{l=0}^{L-1}\left\lfloor \frac{n_l}{n_0} \right\rfloor^{n_0} \right)\sum_{i=0}^{n_0} \binom{n_L}{i}$ by Montúfar et al. \yrcite{montufar2014number}. Following their work, various results on the lower and upper bounds for the maximal number of linear regions of fully-connected ReLU NNs have been obtained \cite{bianchini2014complexity,telgarsky2015representation,poole2016exponential,montufar2017notes,raghu2017expressive,serra2018bounding,croce2018provable,hu2018nearly,serra2018empirical,hanin2019complexity,hanin2019deep}. For example, Arora et al. \yrcite{arora2016understanding} obtained a lower bound $2\sum_{i=0}^{n_0-1}\binom{m-1}{i}n^{L-1}$  for $n_1= 2m$ and $n_2 = n_3 = \ldots = n_L = n$.  Raghu et al. \yrcite{raghu2017expressive} derived an upper bound $\O(n^{Ln_0})$ when $n_1 = n_2 = \ldots = n_L = n$. Montúfar et al.  \yrcite{montufar2017notes} proved an upper bound of  $\prod_{l=1}^{L}\sum_{i=0}^{m_{l}} \binom{n_l}{i}$ where $m_l = \min\{n_0, n_1 , n_2 , \ldots , n_{l-1}  \}$.   Later, these lower and upper bounds were improved by \cite{serra2018bounding}. 
Recently, Hanin et al. \cite{hanin2019complexity,hanin2019deep} studied the average  number of linear regions when the weights range over $\mathbb{R}^{\#weights}$ and derived an upper bound for the expectation of the number of linear regions of ReLU NNs under several mild assumptions. Other studies have replaced the ReLU activation with the maxout activation or piecewise linear functions, and derived several bounds for the number of linear regions in these cases \cite{montufar2014number,hu2018nearly}.

Most studies on the number of linear regions of ReLU NNs assume that the networks are fully-connected. Under this assumption, the problem is equivalent to counting regions of hyperplane arrangements in general position. Thus, one can use a well-established mathematical tool on hyperplane arrangements, Zaslavsky's Theorem \cite{zaslavsky1975facing}, to directly obtain the maximal numbers of linear regions for one-layer fully-connected ReLU NNs, then derive the upper and lower bounds for multi-layer NNs by induction.
Since CNNs are very popular in practice, it is natural to study an analogous problem on the number of linear regions for ReLU CNNs. However, as far as we know, there are no specific results for CNNs so far.
The difficulty is that, although the problem for CNNs can also be translated to counting regions of hyperplane arrangements, usually the corresponding hyperplane arrangements are not in general position for CNNs, as discussed in Section~\ref{sec: one-layer} and the Supplementary Material. Therefore, mathematical tools like Zaslavsky's Theorem cannot be directly applied.

\textbf{Our Contributions.} In this paper, we establish new mathematical tools needed to study hyperplane arrangements (which usually are not in general position) arisen in CNN case, and use them to derive results on the number of linear regions for ReLU CNNs. {\sl To the best of our knowledge, our paper is the first work on calculating the number of linear regions for CNNs.} The main contributions of this work~are:
\begin{itemize}
\item We translate the problem of counting the linear regions of CNNs to a problem on counting the regions of some class of hyperplane arrangements which usually are not in general position, and develop suitable mathematical tools to solve this problem. Through this we provide the exact formula for the maximal number of linear regions of a one-layer ReLU CNN $\mN$ and show that it actually equals the expectation of the number of linear regions when the weights of $\mN$ range over $\mathbb{R}^{\#weights}$. The asymptotic formula for this number is also derived.

\item Furthermore, we derive upper and lower bounds for the number of linear regions of multi-layer ReLU CNNs by induction and the idea of identifying distinct linear regions. 

\item Based on these bounds, we show that deep ReLU CNNs have exponentially more linear regions per parameter than their shallow counterparts under some mild assumptions on the architectures. This means that deep CNNs have more powerful expressivity than shallow ones and thus provides some hints on why CNNs normally perform better as they get deeper. We also show that ReLU CNNs have much more expressivity than the fully-connected ReLU NNs with asymptotically the same number of parameters, input dimension and number of layers.    
\end{itemize}
This paper is organized as follows. We provide a detailed description of the CNN architectures that will be considered throughout the paper, and then introduce the definition of activation patterns and linear regions
in Section \ref{sec: preliminary}. In Section \ref{sec: one-layer}, we obtain the maximal and average numbers of linear regions of one-layer ReLU CNNs in Theorem \ref{th:1}.  In Section \ref{sec: multi-layer}, we derive results on multi-layer ReLU CNNs. A comparison on the expressivity of distinct architectures is given in Section~\ref{sec: comparison}. We briefly explain the experimental settings for verifying our results by sampling methods in Section~\ref{sec: experiment}. In Section \ref{sec: conclusion}, we provide the conclusion and propose future directions.  The preliminary knowledge on hyperplane arrangements and the proofs of Theorems are given in the Supplementary Material.

\section{Preliminary} \label{sec: preliminary}
In this section, we fix some notations and introduce the CNN architecture which will be considered in this paper. Let $\mathbb{N}$, $\mathbb{N}^+$ and $\mathbb{R}$ be the sets of nonnegative integers, positive integers and real numbers, respectively. For a set $S$, let $\#S$ denote the number of elements in $S$.  
In this paper, we consider ReLU CNNs $\mN$ with $L$ hidden convolutional layers (we exclude pooling layers and fully-connected layers, and do not use zero-padding for simplicity).   
Let the dimension of input neurons of $\mN$ be $n_0^{(1)} \times n_0^{(2)} \times d_0$, where $n_0^{(1)},n_0^{(2)},d_0$ are the height, the width and the depth of the input space (we also call the input space the $0$-th layer), respectively. Assume that there are $n_l^{(1)} \times n_l^{(2)} \times d_l$ neurons (i.e., $d_l$ feature maps with the dimension $n_l^{(1)} \times n_l^{(2)}$) in the $l$-th hidden layer for $1\leq l \leq L$. The Rectified Linear Unit (ReLU) is adopted as the activation function for each neuron in the hidden layers. There are $d_l$ filters with dimension $f_l^{(1)}\times f_l^{(2)}\times d_{l-1}$ between neurons in the $(l-1)$-th and the $l$-th hidden layers. Such filters slide from left to right and from top to bottom across feature maps as far as possible with a stride $s_l$. 
We assume that the output layer has only one unit, which is a linear combination of the outputs in the $L$-th hidden layer. As explained in Lemma $2$ from \cite{pascanu2013number}, the number of (linear) output units of a ReLU NN does not affect the number of linear regions that it can realize since the composition of affine functions is still affine. By the same argument this claim is also true for a ReLU CNN. {\sl Therefore, in this paper, we take one unit in the output layer for simplicity and ignore the output layer in the statements of results.}  
Let $X^0=(X^0_{a,b,c})_{n_0^{(1)} \times n_0^{(2)} \times d_0} \in \mathbb{R}^{n_0^{(1)} \times n_0^{(2)} \times d_0}$ be the inputs of $\mN$ and $X^l=(X^l_{a,b,c})_{n_l^{(1)} \times n_l^{(2)} \times d_l} \in \mathbb{R}^{n_l^{(1)} \times n_l^{(2)} \times d_l}$ be the outputs of the $l$-th hidden layer. 
The weights $W=(W^1,W^2,\ldots,W^L)$ and biases $B=(B^1,B^2,\ldots,B^L)$ are drawn from a fixed distribution $\mu$ which has densities with respect to Lebesgue measure in $\mathbb{R}^{\#weights+\#bias}$,  where $W^{l}=(W^{l,1},W^{l,2},\ldots,W^{l,d_l})$ such that  $W^{l,k}=(W_{a,b,c}^{l,k})_{f_l^{(1)}\times f_l^{(2)}\times d_{l-1}} \in \mathbb{R}^{f_l^{(1)}\times f_l^{(2)}\times d_{l-1}}$ is the weight matrix of the $k$-th filter between neurons in the $(l-1)$-th and the $l$-th hidden layers; and $B^{l}=(B^{l,1},B^{l,2},\ldots,B^{l,d_l})\in \mathbb{R}^{d_l}$, such that $B^{l,k} \in \mathbb{R}$ is the bias for the $k$-th filter between neurons in the $(l-1)$-th and the $l$-th hidden layers.
Therefore, for any given weights $W$ and biases $B$, this CNN can been seen as a piece-wise linear function 
$
\mathcal{F}_{\mN,W,B}: \quad 
\mathbb{R}^{n_0^{(1)} \times n_0^{(2)} \times d_0} \rightarrow \mathbb{R}
$
given by 
\begin{align*}
\mathcal{F}_{\mN,W,B}(X^0)=g_{L+1}\circ h_{L}\circ g_{L}\circ\cdots\circ h_{1}\circ g_{1}(X^0),
\end{align*}
where $g_l$ is an affine function and $h_l$ is a ReLU activation function. More specifically, let $Z^l(X^0;\theta)=(Z^l_{i,j,k}(X^0;\theta))_{n_l^{(1)} \times n_l^{(2)} \times d_l} \in \mathbb{R}^{n_l^{(1)} \times n_l^{(2)} \times d_l}$ be the pre-activations of the $l$-th layer, where  $\theta:=\{W,B\}$ is a fixed set of parameters (weights and biases) in the CNN $\mN$. For $1\leq l \leq L$, we have 
\begin{align}\label{eq:convolution_formula2}
& Z^l_{i,j,k}(X^0;\theta)=g_{l}(X^{l-1})  \nonumber
\\&=
\sum_{a=1}^{f_l^{(1)}}\sum_{b=1}^{f_l^{(2)}} \sum_{c=1}^{d_{l-1}}  W_{a,b,c}^{l,k} X^{l-1}_{a+(i-1)s_l,b+(j-1)s_l,c}   + B^{l,k}
\end{align}
and 
\begin{align}\label{eq:convolution_formula3}
X^l_{i,j,k}&=h_{l}(Z^l_{i,j,k}(X^0;\theta))=\max(Z^l_{i,j,k}(X^0;\theta),0).
\end{align}
The following relation between the number of neurons in the $(l-1)$-th and the $l$-th layers are easy to derive. 
\begin{lem}[\cite{dumoulin2016guide}]\label{lem:1}
For $1\leq l \leq L$, we have
$n_l^{(1)} = \lfloor \frac{n_{l-1}^{(1)}-f_l^{(1)}}{s_l} \rfloor+1$
and
$n_l^{(2)} = \lfloor \frac{n_{l-1}^{(2)}-f_l^{(2)}}{s_l} \rfloor+1$,
\newline
where $\left\lfloor x \right\rfloor$ is the greatest integer less~than~or~equal~to~$x$.
\end{lem}

\begin{rem}\label{rem:1}
When the stride $s_l=1$, we have $n_l^{(1)} = n_{l-1}^{(1)}-f_l^{(1)} +1$
and
$n_l^{(2)} = n_{l-1}^{(2)}-f_l^{(2)} +1$. In this case, when a filter slides, all neurons in the $(l-1)$-th hidden layer are involved in the convolutional calculation. 
\end{rem}

By analogy with the ReLU NN case \cite{pascanu2013number,montufar2014number,serra2018bounding,brandfonbrener2018expressive,hanin2019complexity,hanin2019deep}, we introduce the following definition of activation patterns and linear regions for ReLU CNNs.  
\begin{defi}
[Activation Patterns and Linear Regions]\label{def:activation-regions}
\textit{
Let $\mN$ be a ReLU CNN with $L$ hidden convolutional layers given above. 
An activation pattern of $\mN$ is a function $\mP$ from the set of neurons to $\{1,-1\}$, i.e., for each neuron $z$ in $\mN$, we have $\mP(z) \in \{1,-1\}$.   
Let $\theta$ be a fixed set of parameters (weights and biases) in $\mN$, and $\mP$ be an activation pattern. The region corresponding to $\mP$ and $\theta$ is
\begin{align*}
\mR(\mP;\theta) :=
\{
X^0\in \mathbb{R}^{n_0^{(1)} \times n_0^{(2)} \times d_0}:  \\ z(X^0;\theta)\cdot {\mP(z)} >0,\quad \forall z \text{ a neuron in }\mN
\},
\end{align*}
where $z(X^0;\theta)$ is the pre-activation of a neuron $z$. A linear region of $\mN$ at $\theta$ is a non-empty set $\mR(\mP,\theta)\neq \emptyset$ for some activation pattern $\mP$. Let $R_{\mN,\theta}$ denote the number of linear regions of $\mN$ at $\theta$, i.e., 
$
R_{\mN,\theta} := \#\{  \mR(\mP;\theta): \mR(\mP;\theta)\neq \emptyset ~ \text{ for some activation pattern } \mP    \}.
$
Moreover, let $R_{\mN}:=\max_\theta R_{\mN,\theta}$ denote the maximal number of linear regions of~$\mN$ when $\theta$ ranges over $\mathbb{R}^{\#weights+\#bias}$.
}
\end{defi}
\begin{rem}
By the above definition it is easy to check that each non-empty $\mR(\mP;\theta)$ is a convex set. Furthermore, $\mathcal{F}_{\mN,W,B}$ becomes an affine function when restricted to each nonempty linear region $\mR(\mP;\theta)$ of $\mN$. Thus $\mathcal{F}_{\mN,W,B}$ can represent a piecewise linear function with $R_{\mN,\theta}$ linear pieces. Therefore, the number $R_{\mN,\theta}$ of linear regions can be seen as a measure on the expressivity of a CNN. The more linear regions a CNN has, the more complicated functions it can represent. The aim of this paper is to provide a characterization of the number of linear regions for CNNs, and use it to compare the expressivity of different CNNs.  
\end{rem}

By Definition \ref{def:activation-regions}, each activation pattern of a CNN $\mN$ is a function $\mP$ from the set of its neurons to $\{1,-1\}$. It is obvious that there are at most $2^{\# neurons}$ such functions. Therefore, the number of activation patterns is also at most $2^{\# neurons}$. Then, we derive the following trivial upper bound for the number of linear regions of a ReLU CNN. Actually, a similar result for a fully-connected ReLU NN is given in Proposition $3$ from \cite{montufar2014number}.  
\begin{lem}\label{lem:naive_upper_bound}
Let $\mN$ be a ReLU CNN with $n$ hidden neurons. Then, the number $R_{\mN}$ of linear regions of $\mN$ is at most $2^{n}$.
\end{lem}

\section{The Number of Linear Regions for One-Layer CNNs}
\label{sec: one-layer}
In this section, we obtain the exact formula for the maximal number and the average number of linear regions of one-layer CNNs, and derive their asymptotic formulas when the number of filters tends to infinity.

\subsection{Exact Formulas.}
First, we recall the following result on the maximal number of linear regions of one-layer fully-connected ReLU NNs.

\begin{thm} [Proposition $2$ from \cite{pascanu2013number}] \label{thm:one_layer__upper_bound_NN}
Let $\mN$ be a one-layer ReLU NN with $n_0$ input neurons and $n_1$
hidden neurons. Then, the maximal number of linear regions of $\mN$ is equal to
$
\sum_{i=0}^{n_0} \binom{n_1}{i}.
$
\end{thm}

Theorem \ref{thm:one_layer__upper_bound_NN} was derived by translating this problem to a study on the number of regions of hyperplane arrangements in general position, then directly applying a pure mathematical result, Zaslavsky's Theorem \cite{zaslavsky1975facing,Stanley04anintroduction}, which states that, when an arrangement with $n_1$ hyperplanes is in general position, $\mathbb{R}^{n_0}$ can be divided into $\sum_{i=0}^{n_0} \binom{n_1}{i}$ distinct regions.
Basic background on hyperplane arrangements and general position is given in Section $1$ of the Supplementary Material.  
 
Since the set of ReLU CNNs can be seen as a subset of ReLU NNs, Theorem \ref{thm:one_layer__upper_bound_NN} also gives an upper bound for $R_\mN$ where $\mN$ is a one-layer ReLU CNN. However, for the CNN case, usually  this upper bound is not equal to the exact number since the corresponding hyperplane arrangement are not in general position normally. In this paper, we develop new tools to study the number of regions of corresponding hyperplane arrangements (which are not in general position usually) for ReLU CNNs. More precisely, we translate the problem for ReLU CNNs to a tractable integer programming problem by techniques and results from combinatorics and linear algebra. (see Eqs. \eqref{def: K_N}, \eqref{eq:1} and Section~$2$ of the Supplementary Material).    
Our first main result is stated as follows, which shows that the exact number of $R_\mN$ is much smaller than the upper bound given by Theorem~\ref{thm:one_layer__upper_bound_NN} for a one-layer ReLU CNN $\mN$.

\begin{thm}\label{th:1}
Assume that $\mN$ is a one-layer ReLU CNN with input dimension $n_0^{(1)} \times n_0^{(2)} \times d_0$ and hidden layer dimension $n_1^{(1)} \times n_1^{(2)} \times d_1$. The $d_1$ filters have the dimension $f_1^{(1)}\times f_1^{(2)}\times d_0$ and the stride $s_1$.  
Suppose that the parameters $\theta=\{W,B\}$ are drawn from a fixed distribution $\mu$ which has densities with respect to Lebesgue measure in $\mathbb{R}^{\#weights+\#bias}$. 
Define $I_\mN=\{(i,j):~ 1\leq i \leq n_1^{(1)}, ~ 1\leq j \leq n_1^{(2)}\}$ and $S_\mN=(S_{i,j})_{n_1^{(1)}\times n_1^{(2)}}$ where   
\begin{align*}
S_{i,j}=\{ (a+(i & -1)s_1, ~  b+(j-1)s_1,c): 1\leq a \leq f_1^{(1)} ,\\&  1\leq b \leq f_1^{(2)} ,~ 1\leq c \leq d_0\}
\end{align*}
for each $(i,j) \in I_\mN$. Therefore, $S_{i,j}$ is the set of indexes of input neurons involved in the calculation of the pre-activation $Z^1_{i,j,k}(X^0;\theta)$. Furthermore, $\cup_{(i,j) \in I_\mN} S_{i,j} $ is the set of indexes of input neurons involved in the convolutional calculation of $\mathcal{F}_{\mN,W,B}$.
Let 
\begin{align} \label{def: K_N}
& K_{\mN}:  =\{(t_{i,j})_{(i,j)\in I_\mN}:t_{i,j}\in \mathbb{N},~~   \nonumber
\\&
\sum_{(i,j)\in J}t_{i,j}\leq \#\cup_{(i,j) \in J} S_{i,j}~~ \forall J \subseteq I_\mN \}.
\end{align}
Then, we obtain the following two results.

(i) The maximal number $R_\mN$ of linear regions of $\mN$  equals
\begin{align}\label{eq:1}
R_\mN = \sum_{(t_{i,j})_{(i,j)\in I_\mN} \in K_{\mN} }~~\prod_{{(i,j)\in I_\mN}}\binom{d_1}{t_{i,j}}.
\end{align}

(ii) Moreover, Eq. \eqref{eq:1} also equals the expectation of the number $R_{\mN,\theta}$ of linear regions of $\mN$:
\begin{align}\label{eq:expectation}
\mathbb{E}_{\theta \sim \mu} [R_{\mN,\theta}] = \sum_{(t_{i,j})_{(i,j)\in I_\mN} \in K_{\mN} }~~\prod_{{(i,j)\in I_\mN}}\binom{d_1}{t_{i,j}}.
\end{align}
\end{thm}
The detailed proof of Theorem \ref{th:1} is given in the Supplementary Material. We briefly explain the idea below.

\begin{proof}[Outline of the Proof of Theorem \ref{th:1}] First, by Definition \ref{def:activation-regions}, we translate the problem to the calculation of the number of regions of some specific hyperplane arrangements which are not in general position usually. Next, in Proposition $2$ of the Supplementary Material, we derive a generalization of Zaslavsky's Theorem, which can be used to handle a large class of hyperplane arrangements that are not in general position. More specifically, we obtain an upper bound for the number of regions of such hyperplane arrangements and show that if a hyperplane arrangement satisfies the two conditions (i) and (ii) in Proposition $2$, then this upper bound equals the exact number of its regions. The rest of Section $2$ (Lemmas $3$ -- $7$) in the Supplementary Material is devoted to showing that, actually, the specific hyperplane arrangements corresponding to a one-layer ReLU CNN $\mN$ satisfy the conditions for hyperplane arrangements in Proposition $2$. Thus, finally we can apply Proposition $2$ of the Supplementary Material to derive the maximal and average numbers of linear regions for $\mN$. 
\end{proof}

Next, we provide several examples to explain Theorem \ref{th:1}.
\begin{ex}\label{ex:one _layer}
Let $n_0^{(1)} = d_0 = f_1^{(1)} = s_1 = 1$,  $n_0^{(2)} =  3$ and $f_1^{(2)}= 2$ in Theorem \ref{th:1}. Then by Lemma \ref{lem:1} we have $n_1^{(1)} = 1$ and $n_1^{(2)} = 2$.  Furthermore, we obtain $ I_\mN = \{ (1,1), (1,2) \}$, $S_{1,1} = \{ (1,1,1), (1,2,1)\}$,  $S_{2,1} = \{ (1,2,1), (1,3,1) \}$ and 
$ 
K_{\mN}  =\{(t_{1,1},t_{1,2})\in \mathbb{N}^2: ~ t_{1,1}\leq 2, ~ t_{1,2}\leq 2,     ~
t_{1,1} + t_{1,2}\leq 3 \} =\{(0,0),~ (0,1),~ (0,2),~ (1,0),~ (1,1),~ (1,2),~ (2,0),~ (2,1)\}.
$
Finally, by Eq. \eqref{eq:1} 
we derive 
$$
R_\mN = 
\sum_{(t_{1,1},t_{1,2})\in K_{\mN} }~~\binom{d_1}{t_{1,1}} \binom{d_1}{t_{1,2}} = d_1^3+d_1^2+d_1+1,
$$
which is also verified by our experiments for $1\leq d_1\leq 8$ (the experimental settings are given in Section \ref{sec: experiment}).
On the other hand, by Lemma \ref{lem:naive_upper_bound}, we have $R_\mN \leq 2^{2d_1}$; by Theorem \ref{thm:one_layer__upper_bound_NN}, we obtain $R_\mN \leq \sum_{i=0}^{3}\binom{2d_1}{i}$ (since the CNN in Example \ref{ex:one _layer} can be seen as a one-layer NN with $3$ input neurons and $2d_1$ hidden neurons).  When $1\leq d_1\leq 8$, the above bounds for $R_\mN$ are given in Table \ref{tab:1} (more examples are given in Section $5$ of the Supplementary Material). As can be seen, the exact number of $R_\mN$ we obtained in Theorem \ref{th:1} is smaller than the upper bounds obtained by previous methods.  
\end{ex}

\begin{table*}[t!]
\caption{The results for the maximal number of linear regions for a one-layer ReLU CNN $\mN$ with input dimension $1\times 3\times 1$, hidden layer dimension $1\times 2\times d_1$, $d_1$ filters with dimension $1\times 2\times 1$, and stride $s_1=1$.  More precisely, we have $R_\mN = d_1^3+d_1^2+d_1+1.$}
\begin{center}
\resizebox{0.93\linewidth}{!}{
\begin{tabular}{c|c|c|c|c|c|c|c|c}
\hline
&$d_1=1$&$d_1=2$&$d_1=3$&$d_1=4$&$d_1=5$&$d_1=6$&$d_1=7$&$d_1=8$
\\
\hline
$R_\mN$ by Theorem \ref{th:1} & 4&15&40&85&156&259&400&585
\\
\hline
Upper bounds by Theorem \ref{thm:one_layer__upper_bound_NN} & 4&15&42&93&176&299&470&697
\\
\hline
Upper bounds by Lemma $2$ & 4&16&64&256&1024&4096&16384&65536
\\
\hline
\end{tabular}\label{tab:1}
}
\end{center}
\end{table*}

\begin{ex} 
Let  $s_1=1$, $f_1^{(1)}=n_0^{(1)}$ and $f_1^{(2)}=n_0^{(2)}$. Then $n_1^{(1)}=n_1^{(2)}=1$.  Therefore, the CNN becomes a one-layer fully-connected ReLU NN with $d$ input neurons and $d_1$ hidden neurons where $d=n_0^{(1)} \times n_0^{(2)} \times d_0$. Under these assumptions, by Theorem \ref{th:1} we have $I_\mN=\{(1,1)\}$, $K_{\mN}=\{k\in\mathbb{Z}: 0\leq k\leq d \}$. Thus, by \eqref{eq:1} the maximal number of linear regions of a one-layer fully-connected ReLU NN $\mN$ with input dimension $d$ and output dimension $d_1$ equals 
$
R_\mN = \sum_{k=0}^{d}\binom{d_1}{k},
$
which implies the well-known result in Theorem \ref{thm:one_layer__upper_bound_NN} (see   \cite{pascanu2013number,montufar2014number,serra2018bounding,hanin2019complexity,hanin2019deep}). This means that Theorem \ref{th:1} is a more general result than Theorem~\ref{thm:one_layer__upper_bound_NN}. 
\end{ex}

\begin{ex}
Let  $s_1=f_1^{(1)}=f_1^{(2)}=1$, which means that each filter has the dimension $1\times 1 \times d_0$. Then, $S_{i,j}=\{ (i,j,c):~ 1\leq c \leq d_0)\}$ is a $d_0$-element set and thus $\#\cup_{(i,j) \in J} S_{i,j} = d_0\times\#J$ for each $J \subseteq I_\mN$. Therefore,  
$
K_{\mN}
=
\{(t_{i,j})_{n_1^{(1)}\times n_1^{(2)}}: 0 \leq t_{i,j} \leq d_0 \}
=
\{0,1,2,\ldots,d_0\}^{n_1^{(1)}\times n_1^{(2)}}
$
and 
\begin{align*}
R_\mN &= \sum_{(t_{i,j})_{n_1^{(1)}\times n_1^{(2)}} \in \{0,1,2,\ldots,d_0\}^{n_1^{(1)}\times n_1^{(2)}} }~~\prod_{{(i,j)\in I_\mN}}\binom{d_1}{t_{i,j}}.
\end{align*}
When $d_1$ tends to infinity, we obtain 
\begin{align}\label{filter_one}
R_\mN = \left(\frac{ {d_1}^{d_0}}{d_0!}\right)^{n_1^{(1)}\times n_1^{(2)}} + ~~ \O(d_1^{n_0^{(1)}\times n_0^{(2)}\times d_0-1}).
\end{align}
We can see that $R_\mN=  \Theta(d_1^{n_0^{(1)}\times n_0^{(2)}\times d_0})$ in this example. In the following subsection, we will show that this also holds for general cases.
\end{ex}
\subsection{Asymptotic Analysis.}
In this subsection, we study the asymptotic behavior of $R_\mN$.

For two functions $f(n)$ and  $g(n)$, we write $f(n) = \O(g(n))$ if there exists some positive constant $c>0$ such that $f(n) \leq c g(n)$ for all $n$ larger than some constant; $f(n) = \Omega(g(n)) $ if there exists some positive constant $c$ such that $f(n) \geq c g(n)$ for all $n$ large enough; and $f(n) = \Theta(g(n)) $ if there exists some positive constants $c_1,c_2$ such that $c_1 g(n) \leq f(n) \leq c_2 g(n)$ for all $n$ large enough.

We need the following lemma in the asymptotic analysis.

\begin{lem}\label{cor:1}
Let $\mN, I_\mN,K_\mN, S_{i,j}$ be the same as defined in Theorem \ref{th:1}. Then,
there always exists some $(t_{i,j})_{(i,j)\in I_\mN} \in K_{\mN}$ such that 
$$
 \sum_{(i,j)\in I_\mN}t_{i,j}= \#\cup_{(i,j) \in I_\mN} S_{i,j}.
$$
\end{lem}

We derive the following asymptotic formula for $R_\mN$.

\begin{thm} [Asymptotic Analysis] \label{th: asy}
Let $\mN$ be the one-layer ReLU CNN defined in Theorem \ref{th:1}.
Suppose that $ n_0^{(1)},n_0^{(2)},d_0, f_1^{(1)}, f_1^{(2)}, s_1$ are some fixed integers.  When $d_1$ tends to infinity, the asymptotic formula for the maximal number of linear regions of $\mN$ behaves as $R_\mN  =  \Theta(d_1^{ \#\cup_{(i,j) \in I_\mN} S_{i,j}  })$ asymptotically. Furthermore, if all input neurons have been involved in the convolutional calculation, i.e., $\cup_{(i,j) \in I_\mN}  S_{i,j} = \{ (a,b,c): 1\leq a \leq n_0^{(1)} ,~ 
1\leq b \leq n_0^{(2)} ,~ 1\leq c \leq d_0 \}$, we have  
\begin{align} 
R_\mN  =  \Theta(d_1^{n_0^{(1)}\times n_0^{(2)}\times d_0}). 
\end{align} 
\end{thm}

\begin{rem}
Note that, by Theorem \ref{th:1}, $R_\mN$ grows at most as a polynomial of the number of neurons in the hidden layers, instead of growing exponentially fast, as suggested by Lemma~\ref{lem:naive_upper_bound}. This implies that the upper bound in Lemma~\ref{lem:naive_upper_bound} is too loose and may not be achieved in practice.    
\end{rem}

\section{Bounds for the Number of Linear Regions for Multi-layer CNNs}\label{sec: multi-layer}

In this section, we consider multi-layer CNNs and derive the lower and upper bounds for their maximal numbers of linear regions.  
First, we prove a lemma on the composition of two consecutive convolutional layers without an activation function layer between them. 
It is easy to see that such a composition is equivalent to a single convolutional layer. However, we could not find a precise description in the literature of this phenomenon concerning the relation of filter sizes and strides between these three convolutional layers. Therefore, we precisely describe this phenomenon and prove it in the following theorem. 
\begin{thm}\label{composition_convolutional}
Let $\mN$ be a two-layer CNN without activation layers. For $l=1,2$, there are $d_l$ filters with dimension $f_l^{(1)}\times f_l^{(2)}\times d_{l-1}$ and stride $s_l$ between neurons in the $(l-1)$-th and the $l$-th hidden layers. Then $\mN$ can be realized as a CNN with only one hidden convolutional layer such that its $d_2$ filters have size $f^{(1)}\times f^{(2)}\times d_{0}= (f_1^{(1)}+(f_2^{(1)}-1)s_1)\times (f_1^{(2)}+(f_2^{(2)}-1)s_1)\times d_{0} $ and stride $s=s_1s_2$.  
In particular, if $f_1^{(1)}=f_1^{(2)}=s_1=1$, we have $f^{(1)}\times f^{(2)}\times d_{0}= f_2^{(1)}\times f_2^{(2)}\times d_{0} $ and stride $s=s_2$. That is, if the first  convolutional layer has filter size $1\times 1\times d_0$ and stride $1$, then the composition of the two convolutional layers has the same filter size and stride as the second convolutional layer. 
\end{thm}

Now we are ready to derive the lower and upper bounds for the maximal numbers of linear regions of multi-layer CNNs using induction and the idea of identifying distinct linear regions motivated by  \cite{pascanu2013number,montufar2014number}.
\begin{thm}\label{th:multilayer}
Suppose that $\mN$ is a ReLU CNN with $L$ hidden convolutional layers. The input dimension is $n_0^{(1)} \times n_0^{(2)} \times d_0$; the $l$-th hidden layer has dimension $n_l^{(1)} \times n_l^{(2)} \times d_l$ for $1\leq l \leq L$; and there are $d_l$ filters with dimension $f_l^{(1)}\times f_l^{(2)}\times d_{l-1}$ and stride $s_l$ in the $l$-th layer. Assume that $d_l\geq d_0$ for each $1\leq l \leq L$. Then, we have

(i)  The maximal number $R_\mN$ of linear regions of $\mN$ is at least (lower bound)
\begin{align}\label{eq:10}
R_\mN \geq  R_{\mN'}\prod_{l=1}^{L-1}\left\lfloor\frac{d_l}{d_0}\right\rfloor^{n_l^{(1)} \times n_l^{(2)}\times d_0},
\end{align}
where  $\mN'$ is a one-layer ReLU CNN which has input dimension $n_{L-1}^{(1)} \times n_{L-1}^{(2)} \times d_0$ (the third dimension is $d_0$, not $d_{L-1}$), hidden layer dimension $n_L^{(1)} \times n_L^{(2)} \times d_L$, and $d_L$ filters with dimension $f_L^{(1)}\times f_L^{(2)}\times d_{0}$ and stride $s_L$. Note that the exact formula of $R_{\mN'}$ can be calculated by Eq. \eqref{eq:1}.

(ii)  The maximal number $R_\mN$ of linear regions of $\mN$ is at most (upper bound)
\begin{align}\label{eq:100}
R_\mN \leq
 R_{\mN''}\prod_{l=2}^{L}\sum_{i=0}^{n_0^{(1)} n_0^{(2)} d_0} \binom{n_l^{(1)}  n_l^{(2)}   d_l}{i},
\end{align}
where $\mN''$ is a one-layer ReLU CNN which has input dimension $n_{0}^{(1)} \times n_{0}^{(2)} \times d_0$, hidden layer dimension $n_1^{(1)} \times n_1^{(2)} \times d_1$, and  $d_1$ filters with dimension $f_1^{(1)}\times f_1^{(2)}\times d_{0}$ and stride~$s_1$.
\end{thm}

\begin{table*}[t!]
\caption{The upper and lower bounds for $R_\mN$ in Example \ref{ex:two layer}.}
\begin{center}
\resizebox{0.95\linewidth}{!}{
\begin{tabular}{c|c|c|c|c|c|c|c|c}
\hline
&$d_2=1$&$d_2=2$&$d_2=3$&$d_2=4$&$d_2=5$&$d_2=6$&$d_2=7$&$d_2=8$
\\
\hline
Upper bounds by Theorem \ref{th:multilayer} &220 & 880 & 3520 & 13585 & 46640 & 138050 & 356180 & 819115 
\\
\hline
Estimation of $R_\mN$ by sampling methods &170 & 261 & 685 &1186 &1796 &2725 &3398 &4822 
\\
\hline
Lower bounds by Theorem \ref{th:multilayer} & 32 & 120 & 320 & 680 & 1248 & 2072 & 3200 & 4680 
\\
\hline
\end{tabular}\label{tab:2}
}
\end{center}
\end{table*}
\begin{ex}\label{ex:two layer}
Let $\mN$ be a two-layer CNN such that the input dimension is $1\times 4\times 1$, there are $2$ filters with dimension $1\times 2\times 1$ and stride $1$ in the first hidden layer; and $d_2$ filters with dimension $1\times 2\times 2$ and stride $1$ in the second hidden layer. The dimensions of neurons in the first and second hidden layer are $1\times 3\times 2$ and $1\times 2\times d_2$ respectively. Theorem \ref{th:multilayer} yields the upper and lower bounds for $R_\mN$ as shown in Table \ref{tab:2}, which is compatible with the estimation of $R_\mN$ by sampling methods in our experiment. 
\end{ex}

\begin{ex} [Reduce to fully-connected ReLU NN case] 
Let $n_0^{(1)}=n_0^{(2)}=1$ and $s_l=f_l^{(1)}=f_l^{(2)}=n_l^{(1)}=n_l^{(2)}=1$ for each $1\leq l \leq L$. Then the CNN becomes a fully-connected ReLU NN. Under these assumptions, by Eq. \eqref{eq:1} and Theorem \ref{th:multilayer} we have 
\begin{align}\label{eq:upper_lower_bounds_NN}
 \sum_{k=0}^{d_{0}}\binom{d_L}{k} \times 
\prod_{l=1}^{L-1}\left\lfloor\frac{d_l}{d_0}\right\rfloor^{ d_0} \leq R_\mN \leq \prod_{l=1}^{L}\sum_{i=0}^{d_0} \binom{d_l}{i},
\end{align}
which is a well-known result for fully-connected ReLU NNs (see Theorem $4$ from \cite{montufar2014number} for the first inequality; see Proposition $3$ from \cite{montufar2017notes},   Theorem $1$ from \cite{raghu2017expressive} and Theorem $1$ from \cite{serra2018bounding} for the second inequality). Note that \eqref{eq:upper_lower_bounds_NN} implies that  $R_\mN=\Theta(d^{Ld_0})$ when $d_0$ is fixed and $d_1=d_2=\ldots = d_L =d \rightarrow +\infty$.  
\end{ex}

\section{Expressivity Comparison of Different Network Architectures}\label{sec: comparison}
In this section, we compare the expressivity of different network architectures in terms of the maximal number of linear regions based on the explicit formulas and bounds derived in Sections \ref{sec: one-layer} and \ref{sec: multi-layer}. The first conclusion is that deep CNNs usually have more expressivity than their shallow counterparts with the same number of parameters.
Furthermore, we compare ReLU CNNs with the fully-connected ReLU NNs with asymptotically the same number of parameters,  input dimension and  number of layers.
We show that CNNs have more expressivity than fully-connected NNs in this setting. 
\subsection{Deep CNNs v.s. Shallow CNNs}\label{subsec: compare_deep_shallow_CNN}
First, we calculate the number of parameters for CNNs. 
\begin{lem}
\label{prop:number_params_1}
Let $\mN$ be an $L$-layer ReLU CNN in Theorem \ref{th:multilayer} (ignoring the output layer for simplicity). Then, the number of parameters in $\mN$ is
$
 \sum_{l=1}^L \left(f_l^{(1)}\times f_l^{(2)}\times d_{l-1}\times d_l + d_l\right).
$
\end{lem}

Now we can derive the number of linear regions per parameter for deep and shallow CNNs. The next result follows directly from Theorem \ref{th: asy}, Theorem \ref{th:multilayer}  and Lemma~\ref{prop:number_params_1}.

\begin{thm}\label{th:asy_compare}
Let $\mN_1$ be an $L$-layer ReLU CNN in Theorem~\ref{th:multilayer} where $f_l^{(1)}$, $f_l^{(2)}=\O(1)$ for $1\leq l \leq L$, and $d_0=\O(1)$. When $d_1=d_2=\cdots=d_L=d$ tends to infinity, we obtain that $\mN_1$ has $\Theta (Ld^2)$ parameters, and the ratio of $R_{\mN_1}$ to the number of parameters of $\mN_1$ is 
$$
\frac{R_{\mN_1}}{\# \text{ parameters of } \mN_1}= 
\Omega \Bigl(\frac{1}{L} \cdot \left\lfloor\frac{d}{d_0}\right\rfloor^{  d_0 \sum_{l=1}^{L-1} n_l^{(1)} n_l^{(2)} -2  } \Bigr).
$$
For a one-layer ReLU CNN ${\mN_2}$ with input dimension $n_0^{(1)} \times n_0^{(2)} \times d_0$ and hidden layer dimension $n_1^{(1)} \times n_1^{(2)} \times Ld^2$, when $Ld^2$ tends to infinity, $\mN_2$ has $\Theta (Ld^2)$ parameters, and the ratio for $\mN_2$ is 
$$
\frac{R_{\mN_2}}{\# \text{ parameters of } \mN_2}= \O\left( \left(Ld^2 \right)^ { d_0 n_0^{(1)} n_0^{(2)} -1  }    \right).
$$
\end{thm}

By Theorem \ref{th:asy_compare} we will show that, with  asymptotically the same  number $\Theta(Ld^2)$ of parameters and the same number of input dimensions $n^2d_0$, deep CNNs can represent functions that have more number of linear regions than shallow CNNs. For simplicity, we set the stride $s_l=1$ for each layer, $n_0^{(1)} = n_0^{(2)}=n$ and $f_l^{(1)}=f_l^{(2)}=1$ in Theorem \ref{th:asy_compare} (in practice, filters with small sizes such as $3\times 3$, $5\times 5$ and $7\times 7$ are often adopted; for such cases, the conclusion is similar to the case $f_l^{(1)}=f_l^{(2)}=1$) in Theorem \ref{th:asy_compare}. Therefore, by Lemma \ref{lem:1} we have $n_l^{(1)} = n_l^{(2)} = n$ for each $1\leq l \leq L$. Then, the first ratio in Theorem \ref{th:asy_compare} is 
$$
\frac{R_{\mN_1}}{\# \text{ parameters of } \mN_1}= 
\Omega \left(\frac{1}{L} \cdot \left\lfloor\frac{d}{d_0}\right\rfloor^{  d_0 (L-1) n^2  -2  } \right),
$$
which grows at least exponentially fast with the number $L$ of hidden layers and polynomially fast with the depth $d$ of each hidden layer.

In contrast, the second ratio in Theorem \ref{th:asy_compare} grows at most polynomially fast with $L$ and $d$:
$$
\frac{R_{\mN_2}}{\# \text{ parameters of } \mN_2}= \O \left( (Ld^2)^ { d_0 n^{2} -1  }    \right).
$$
Therefore, we obtain that $R_{\mN_1}$ is far larger than $R_{\mN_2}$ when $L$ and $d$ are large enough.
By this we conclude that {\sl ReLU deep CNNs have much more expressivity than their shallow counterparts with asymptotically the same number of parameters and the same number of input dimensions}. 

\subsection{CNNs v.s. Fully-connected NNs}\label{subsec: compare_CNN_NN}
In this subsection, we compare the expressivity of ReLU CNNs and  fully-connected ReLU NNs with asymptotically the same number of parameters, input dimension and number of hidden layers. The settings for the $L$-layer ReLU CNN $\mN_1$ is the same as in Subsection \ref{subsec: compare_deep_shallow_CNN}. For an $L$-layer fully-connected ReLU NN $\mN_3$, we assume that the input dimension equals $n^2d_0$, and the number of neurons in each of the $L$ hidden layers equals $d_0$. Then $\mN_1$ and $\mN_3$ have asymptotically the  same number of parameters $\O(Ld^2)$, input dimension $n^2d_0$ and  number $L$ of layers. However, the maximal number of linear regions for $\mN_1$ is   
$$
R_{\mN_1}=  \Omega\left(  \left\lfloor\frac{d}{d_0}\right\rfloor^{ L d_0  n^2  }    \right)
$$
by \eqref{filter_one} and Theorem \ref{th:multilayer}. 
On the other hand, for $\mN_3$ we obtain 
\begin{align}
R_{\mN_3}&= \O \left( \binom{d}{n^2d_0}^{ L   }      \right) = \O  \left(\frac{d^{ L d_0  n^2  } }{(n^2d_0)!^L}   \right)
\nonumber
\\&
= \O  \left(\frac{d^{ L d_0  n^2  } }{(\sqrt{2\pi n^2d_0})^L (n^2d_0/e)^ {L d_0  n^2 } }   \right)     
\end{align}
by \eqref{eq:upper_lower_bounds_NN} and the Stirling's formula \cite{flajolet2009analytic}.
Therefore, 
$$
\frac{R_{\mN_1}}{R_{\mN_3}} \geq  \Omega\left(  (\sqrt{2\pi n^2d_0})^L (n^2/e)^ {L d_0  n^2 }   \right).
$$
When $n$ tends to infinity, the ratio $\frac{R_{\mN_1}}{R_{\mN_3}}$ also 
tends to infinity. Thus $R_{\mN_1}$ is much larger than $R_{\mN_3}$, and we conclude that {\sl ReLU CNNs have much more expressivity than the fully-connected ReLU NNs with asymptotically the same number of parameters, input dimension and number of layers}. 

\section{Experimental Settings} \label{sec: experiment}
We empirically validate our results by randomly sampling data points from the input space and determining which linear regions they belong to by Definition \ref{def:activation-regions}.
For a given CNN architecture, we initialize the parameters (weights and biases) based on the He initialization \cite{he2015delving}. Given the sampled weight, each data point in the input space is sampled from a normal distribution with mean $0$ and standard deviation $v$. 
We use $v$ ranging from $\{3, 5, 7, 9, 11, 13\}$ 
and report the maximal number of linear regions from such $v$. We sample $2\times 10^9$ data points in total, and for each data point, we determine which region it belongs to based on Definition \ref{def:activation-regions} (for a new data point $X^0$, we simply calculate the sign of $z(X^0,\theta)$ for each neuron $z$ and use it to determine whether $X^0$ belongs to a new region). This sampling method may skip some regions. Thus, the number of linear regions obtained by sampling is usually smaller than the exact number. However, when the number of sampling points is large enough, we can usually find almost all the linear regions. For example, we use this sampling method to find all $R_\mN$ linear regions for one-layer CNNs $\mN$ in Table \ref{tab:1}, and find the number of regions between the lower and upper bounds for two-layer CNNs in Table \ref{tab:2}.  By these, we validate the correctness of our results. {\sl We provide the codes for the experiments in the Supplementary Material}.

\section{Conclusion and Future Work}\label{sec: conclusion}
In this paper, we obtained exact formulas for the maximal and average number of linear regions of one-layer ReLU CNNs, and derived lower and upper bounds for multi-layer CNNs. By these results, we concluded that deep ReLU CNNs have more expressivity than their shallow counterparts, while ReLU CNNs have more expressivity than fully-connected ReLU NNs per parameter. 

To the best of our knowledge, our paper is the first work investigating the number of linear regions for CNNs. We plan to explore more aspects in the future based on this work. Possible future directions are summarized below. 

(1) In this paper, we only consider ReLU CNNs without pooling layers, fully-connected layers, and zero-padding for simplicity. After adding pooling layers, the functions represented by ReLU CNNs are still piecewise linear, thus the definition of linear regions still applies. It would be interesting to study the number of linear regions for ReLU CNNs with pooling layers, fully-connected layers, and zero-padding in the future.

(2) In Theorem \ref{th:1} we showed that the expectation of 
$R_{\mN,\theta}$ is equal to the maximal number $R_{\mN}$ for a one-layer ReLU CNN $\mN$. This result is consistent with the one-layer fully-connected NN case in \cite{hanin2019deep} (see the last two sentences of Section 2 and the first sentence in Remark 1 of \cite{hanin2019deep}). When the number of layers of the  fully-connected NN is at least two, it is proved in \cite{hanin2019deep} that the expectation of the number of linear regions is much smaller than the maximal number.  
It would be interesting to explore similar formulas for the expectation of 
$R_{\mN,\theta}$ for multi-layer ReLU CNNs.  
We believe this will be a more challenging topic than the fully-connected NN case due to the correlated weights and bias for CNNs.

(3) In Theorem \ref{th:multilayer} we derive lower and upper bounds for multi-layer CNNs. By Table \ref{tab:2} we can see that these bounds are not very close to each other. We would like to derive tighter bounds, or further exact formulas in the future.

(4) We would like to extend our method to study the changing number of linear regions for CNNs when the parameters are updated (for example, by backpropagation) with a small perturbation. When the parameters $\theta$ are replaced by some $\theta + \Delta\theta $, what is the relation between $R_{\mN,\theta}$ and $R_{\mN,\theta + \Delta\theta}$? For which parameters $\theta_1$ and $\theta_2$, the numbers $R_{\mN,\theta_1}$ and $R_{\mN,\theta_2}$ are equal to each other?  During the training process, the parameters $\theta$ changes to some $\theta + \Delta\theta $. Thus, the answer to the above question may help us have a better understanding of the training process and optimization for CNNs. In fact, recently, Hanin and Rolnick \cite{hanin2019complexity,hanin2019deep} have already done some research on the changing number of linear regions during the training process for fully-connected NNs.

(5) In \cite{hu2018nearly}, the ReLU activation was generalized to piecewise linear (PWL) functions. The results on exact formulas and bounds for the number of linear regions for fully-connected PWL NNs were presented. In the future, we plan to replace the ReLU activation with PWL functions for CNNs and study their numbers of linear regions.

\section*{Acknowledgements}
We really appreciate the valuable suggestions given by reviewers for improving the overall quality of this paper. We also would like to thank Dr. Jingtao Zang and Dr. Keqian Yan for helpful discussions.


\bibliography{number_region_CNN}

\begin{thebibliography}{40}
\providecommand{\natexlab}[1]{#1}
\providecommand{\url}[1]{\texttt{#1}}
\expandafter\ifx\csname urlstyle\endcsname\relax
  \providecommand{\doi}[1]{doi: #1}\else
  \providecommand{\doi}{doi: \begingroup \urlstyle{rm}\Url}\fi

\bibitem[Abdel-Hamid et~al.(2014)Abdel-Hamid, Mohamed, Jiang, Deng, Penn, and
  Yu]{abdel2014convolutional}
Abdel-Hamid, O., Mohamed, A.-r., Jiang, H., Deng, L., Penn, G., and Yu, D.
\newblock Convolutional neural networks for speech recognition.
\newblock \emph{IEEE/ACM Transactions on audio, speech, and language
  processing}, 22\penalty0 (10):\penalty0 1533--1545, 2014.

\bibitem[Arora et~al.(2016)Arora, Basu, Mianjy, and
  Mukherjee]{arora2016understanding}
Arora, R., Basu, A., Mianjy, P., and Mukherjee, A.
\newblock Understanding deep neural networks with rectified linear units.
\newblock \emph{arXiv preprint arXiv:1611.01491}, 2016.

\bibitem[Barron(1994)]{barron1994approximation}
Barron, A.~R.
\newblock Approximation and estimation bounds for artificial neural networks.
\newblock \emph{Machine learning}, 14\penalty0 (1):\penalty0 115--133, 1994.

\bibitem[Bianchini \& Scarselli(2014)Bianchini and
  Scarselli]{bianchini2014complexity}
Bianchini, M. and Scarselli, F.
\newblock On the complexity of neural network classifiers: A comparison between
  shallow and deep architectures.
\newblock \emph{IEEE transactions on neural networks and learning systems},
  25\penalty0 (8):\penalty0 1553--1565, 2014.

\bibitem[Brandfonbrener(2018)]{brandfonbrener2018expressive}
Brandfonbrener, D.
\newblock The expressive power of neural networks.
\newblock \emph{CPSC 490}, 2018.

\bibitem[Croce et~al.(2018)Croce, Andriushchenko, and Hein]{croce2018provable}
Croce, F., Andriushchenko, M., and Hein, M.
\newblock Provable robustness of relu networks via maximization of linear
  regions.
\newblock \emph{arXiv preprint arXiv:1810.07481}, 2018.

\bibitem[Cybenko(1989)]{cybenko1989approximation}
Cybenko, G.
\newblock Approximation by superpositions of a sigmoidal function.
\newblock \emph{Mathematics of control, signals and systems}, 2\penalty0
  (4):\penalty0 303--314, 1989.

\bibitem[Dumoulin \& Visin(2016)Dumoulin and Visin]{dumoulin2016guide}
Dumoulin, V. and Visin, F.
\newblock A guide to convolution arithmetic for deep learning.
\newblock \emph{arXiv preprint arXiv:1603.07285}, 2016.

\bibitem[Flajolet \& Sedgewick(2009)Flajolet and
  Sedgewick]{flajolet2009analytic}
Flajolet, P. and Sedgewick, R.
\newblock \emph{Analytic combinatorics}.
\newblock cambridge University press, 2009.

\bibitem[Funahashi(1989)]{funahashi1989approximate}
Funahashi, K.-I.
\newblock On the approximate realization of continuous mappings by neural
  networks.
\newblock \emph{Neural networks}, 2\penalty0 (3):\penalty0 183--192, 1989.

\bibitem[Glorot et~al.(2011)Glorot, Bordes, and Bengio]{glorot2011deep}
Glorot, X., Bordes, A., and Bengio, Y.
\newblock Deep sparse rectifier neural networks.
\newblock In \emph{Proceedings of the fourteenth international conference on
  artificial intelligence and statistics}, pp.\  315--323, 2011.

\bibitem[Goodfellow et~al.(2013)Goodfellow, Warde-Farley, Mirza, Courville, and
  Bengio]{goodfellow2013maxout}
Goodfellow, I.~J., Warde-Farley, D., Mirza, M., Courville, A., and Bengio, Y.
\newblock Maxout networks.
\newblock \emph{arXiv preprint arXiv:1302.4389}, 2013.

\bibitem[Hahnloser \& Seung(2001)Hahnloser and Seung]{hahnloser2001permitted}
Hahnloser, R.~H. and Seung, H.~S.
\newblock Permitted and forbidden sets in symmetric threshold-linear networks.
\newblock In \emph{Advances in neural information processing systems}, pp.\
  217--223, 2001.

\bibitem[Hahnloser et~al.(2000)Hahnloser, Sarpeshkar, Mahowald, Douglas, and
  Seung]{hahnloser2000digital}
Hahnloser, R.~H., Sarpeshkar, R., Mahowald, M.~A., Douglas, R.~J., and Seung,
  H.~S.
\newblock Digital selection and analogue amplification coexist in a
  cortex-inspired silicon circuit.
\newblock \emph{Nature}, 405\penalty0 (6789):\penalty0 947, 2000.

\bibitem[Hanin(2017)]{hanin2017universal}
Hanin, B.
\newblock Universal function approximation by deep neural nets with bounded
  width and relu activations.
\newblock \emph{arXiv preprint arXiv:1708.02691}, 2017.

\bibitem[Hanin \& Rolnick(2019{\natexlab{a}})Hanin and
  Rolnick]{hanin2019complexity}
Hanin, B. and Rolnick, D.
\newblock Complexity of linear regions in deep networks.
\newblock In \emph{International Conference on Machine Learning}, pp.\
  2596--2604, 2019{\natexlab{a}}.

\bibitem[Hanin \& Rolnick(2019{\natexlab{b}})Hanin and Rolnick]{hanin2019deep}
Hanin, B. and Rolnick, D.
\newblock Deep relu networks have surprisingly few activation patterns.
\newblock In \emph{Advances in Neural Information Processing Systems}, pp.\
  359--368, 2019{\natexlab{b}}.

\bibitem[Hanin \& Sellke(2017)Hanin and Sellke]{hanin2017approximating}
Hanin, B. and Sellke, M.
\newblock Approximating continuous functions by relu nets of minimal width.
\newblock \emph{arXiv preprint arXiv:1710.11278}, 2017.

\bibitem[He et~al.(2015)He, Zhang, Ren, and Sun]{he2015delving}
He, K., Zhang, X., Ren, S., and Sun, J.
\newblock Delving deep into rectifiers: Surpassing human-level performance on
  imagenet classification.
\newblock In \emph{Proceedings of the IEEE international conference on computer
  vision}, pp.\  1026--1034, 2015.

\bibitem[He et~al.(2016)He, Zhang, Ren, and Sun]{he2016deep}
He, K., Zhang, X., Ren, S., and Sun, J.
\newblock Deep residual learning for image recognition.
\newblock In \emph{Proceedings of the IEEE conference on computer vision and
  pattern recognition}, pp.\  770--778, 2016.

\bibitem[Hinton et~al.(2012)Hinton, Deng, Yu, Dahl, Mohamed, Jaitly, Senior,
  Vanhoucke, Nguyen, Kingsbury, et~al.]{hinton2012deep}
Hinton, G., Deng, L., Yu, D., Dahl, G., Mohamed, A.-r., Jaitly, N., Senior, A.,
  Vanhoucke, V., Nguyen, P., Kingsbury, B., et~al.
\newblock Deep neural networks for acoustic modeling in speech recognition.
\newblock \emph{IEEE Signal processing magazine}, 29, 2012.

\bibitem[Hornik(1991)]{hornik1991approximation}
Hornik, K.
\newblock Approximation capabilities of multilayer feedforward networks.
\newblock \emph{Neural networks}, 4\penalty0 (2):\penalty0 251--257, 1991.

\bibitem[Hu \& Zhang(2018)Hu and Zhang]{hu2018nearly}
Hu, Q. and Zhang, H.
\newblock Nearly-tight bounds on linear regions of piecewise linear neural
  networks.
\newblock \emph{arXiv preprint arXiv:1810.13192}, 2018.

\bibitem[Krizhevsky et~al.(2012)Krizhevsky, Sutskever, and
  Hinton]{krizhevsky2012imagenet}
Krizhevsky, A., Sutskever, I., and Hinton, G.~E.
\newblock Imagenet classification with deep convolutional neural networks.
\newblock In \emph{Advances in neural information processing systems}, pp.\
  1097--1105, 2012.

\bibitem[Lu et~al.(2017)Lu, Pu, Wang, Hu, and Wang]{lu2017expressive}
Lu, Z., Pu, H., Wang, F., Hu, Z., and Wang, L.
\newblock The expressive power of neural networks: A view from the width.
\newblock In \emph{Advances in neural information processing systems}, pp.\
  6231--6239, 2017.

\bibitem[Mont{\'u}far(2017)]{montufar2017notes}
Mont{\'u}far, G.
\newblock Notes on the number of linear regions of deep neural networks.
\newblock \emph{Sampling Theory Appl., Tallinn, Estonia, Tech. Rep}, 2017.

\bibitem[Montufar et~al.(2014)Montufar, Pascanu, Cho, and
  Bengio]{montufar2014number}
Montufar, G.~F., Pascanu, R., Cho, K., and Bengio, Y.
\newblock On the number of linear regions of deep neural networks.
\newblock In \emph{Advances in neural information processing systems}, pp.\
  2924--2932, 2014.

\bibitem[Pascanu et~al.(2013)Pascanu, Montufar, and Bengio]{pascanu2013number}
Pascanu, R., Montufar, G., and Bengio, Y.
\newblock On the number of response regions of deep feed forward networks with
  piece-wise linear activations.
\newblock \emph{arXiv preprint arXiv:1312.6098}, 2013.

\bibitem[Poole et~al.(2016)Poole, Lahiri, Raghu, Sohl-Dickstein, and
  Ganguli]{poole2016exponential}
Poole, B., Lahiri, S., Raghu, M., Sohl-Dickstein, J., and Ganguli, S.
\newblock Exponential expressivity in deep neural networks through transient
  chaos.
\newblock In \emph{Advances in neural information processing systems}, pp.\
  3360--3368, 2016.

\bibitem[Raghu et~al.(2017)Raghu, Poole, Kleinberg, Ganguli, and
  Dickstein]{raghu2017expressive}
Raghu, M., Poole, B., Kleinberg, J., Ganguli, S., and Dickstein, J.~S.
\newblock On the expressive power of deep neural networks.
\newblock In \emph{Proceedings of the 34th International Conference on Machine
  Learning-Volume 70}, pp.\  2847--2854. JMLR. org, 2017.

\bibitem[Sainath et~al.(2013)Sainath, Mohamed, Kingsbury, and
  Ramabhadran]{sainath2013deep}
Sainath, T.~N., Mohamed, A.-r., Kingsbury, B., and Ramabhadran, B.
\newblock Deep convolutional neural networks for lvcsr.
\newblock In \emph{2013 IEEE international conference on acoustics, speech and
  signal processing}, pp.\  8614--8618. IEEE, 2013.

\bibitem[Serra \& Ramalingam(2018)Serra and Ramalingam]{serra2018empirical}
Serra, T. and Ramalingam, S.
\newblock Empirical bounds on linear regions of deep rectifier networks.
\newblock \emph{arXiv preprint arXiv:1810.03370}, 2018.

\bibitem[Serra et~al.(2018)Serra, Tjandraatmadja, and
  Ramalingam]{serra2018bounding}
Serra, T., Tjandraatmadja, C., and Ramalingam, S.
\newblock Bounding and counting linear regions of deep neural networks.
\newblock In \emph{International Conference on Machine Learning}, pp.\
  4565--4573, 2018.

\bibitem[Silver et~al.(2016)Silver, Huang, Maddison, Guez, Sifre, Van
  Den~Driessche, Schrittwieser, Antonoglou, Panneershelvam, Lanctot,
  et~al.]{silver2016mastering}
Silver, D., Huang, A., Maddison, C.~J., Guez, A., Sifre, L., Van Den~Driessche,
  G., Schrittwieser, J., Antonoglou, I., Panneershelvam, V., Lanctot, M.,
  et~al.
\newblock Mastering the game of go with deep neural networks and tree search.
\newblock \emph{nature}, 529\penalty0 (7587):\penalty0 484, 2016.

\bibitem[Simonyan \& Zisserman(2014)Simonyan and Zisserman]{simonyan2014very}
Simonyan, K. and Zisserman, A.
\newblock Very deep convolutional networks for large-scale image recognition.
\newblock \emph{arXiv preprint arXiv:1409.1556}, 2014.

\bibitem[Stanley(2004)]{Stanley04anintroduction}
Stanley, R.~P.
\newblock An introduction to hyperplane arrangements.
\newblock In \emph{Lecture notes, IAS/Park City Mathematics Institute}, 2004.

\bibitem[Szegedy et~al.(2015)Szegedy, Liu, Jia, Sermanet, Reed, Anguelov,
  Erhan, Vanhoucke, and Rabinovich]{szegedy2015going}
Szegedy, C., Liu, W., Jia, Y., Sermanet, P., Reed, S., Anguelov, D., Erhan, D.,
  Vanhoucke, V., and Rabinovich, A.
\newblock Going deeper with convolutions.
\newblock In \emph{Proceedings of the IEEE conference on computer vision and
  pattern recognition}, pp.\  1--9, 2015.

\bibitem[Telgarsky(2015)]{telgarsky2015representation}
Telgarsky, M.
\newblock Representation benefits of deep feedforward networks.
\newblock \emph{arXiv preprint arXiv:1509.08101}, 2015.

\bibitem[Telgarsky(2016)]{telgarsky2016benefits}
Telgarsky, M.
\newblock Benefits of depth in neural networks.
\newblock \emph{arXiv preprint arXiv:1602.04485}, 2016.

\bibitem[Zaslavsky(1975)]{zaslavsky1975facing}
Zaslavsky, T.
\newblock \emph{Facing Up to Arrangements: Face-Count Formulas for Partitions
  of Space by Hyperplanes}.
\newblock Number 154 in Memoirs of the American Mathematical Society. American
  Mathematical Society, 1975.
\newblock URL \url{https://books.google.ae/books?id=2DUZAQAAIAAJ}.

\end{thebibliography}


\begin{thebibliography}{4}
\providecommand{\natexlab}[1]{#1}
\providecommand{\url}[1]{\texttt{#1}}
\expandafter\ifx\csname urlstyle\endcsname\relax
  \providecommand{\doi}[1]{doi: #1}\else
  \providecommand{\doi}{doi: \begingroup \urlstyle{rm}\Url}\fi

\bibitem[Lojasiewicz(1964)]{lojasiewicz1964triangulation}
Lojasiewicz, S.
\newblock Triangulation of semi-analytic sets.
\newblock \emph{Annali della Scuola Normale Superiore di Pisa-Classe di
  Scienze}, 18\penalty0 (4):\penalty0 449--474, 1964.

\bibitem[Shafarevich \& Remizov(2012)Shafarevich and
  Remizov]{shafarevich2012linear}
Shafarevich, I.~R. and Remizov, A.~O.
\newblock \emph{Linear algebra and geometry}.
\newblock Springer Science \& Business Media, 2012.

\bibitem[Stanley(2004)]{Stanley04anintroduction}
Stanley, R.~P.
\newblock An introduction to hyperplane arrangements.
\newblock In \emph{Lecture notes, IAS/Park City Mathematics Institute}, 2004.

\bibitem[Zaslavsky(1975)]{zaslavsky1975facing}
Zaslavsky, T.
\newblock \emph{Facing Up to Arrangements: Face-Count Formulas for Partitions
  of Space by Hyperplanes}.
\newblock Number 154 in Memoirs of the American Mathematical Society. American
  Mathematical Society, 1975.
\newblock URL \url{https://books.google.ae/books?id=2DUZAQAAIAAJ}.

\end{thebibliography}
\bibliographystyle{icml2020}

\end{document}


\onecolumn

\icmltitle{Supplementary Material for ``On the Number of Linear Regions of Convolutional Neural Networks"}



\begin{icmlauthorlist}
\icmlauthor{Huan Xiong}{mbzuai}
\icmlauthor{Lei Huang}{iiai}
\icmlauthor{Mengyang Yu}{iiai}
\icmlauthor{Li Liu}{iiai}
\icmlauthor{Fan Zhu}{iiai}
\icmlauthor{Ling Shao}{mbzuai,iiai}
\end{icmlauthorlist}

\icmlaffiliation{iiai}{Inception Institute of Artificial Intelligence,
  Abu Dhabi, UAE}
\icmlaffiliation{mbzuai}{Mohamed bin Zayed University of Artificial Intelligence, UAE}

\icmlcorrespondingauthor{Huan Xiong}{huan.xiong@mbzuai.ac.ae}

\icmlkeywords{Machine Learning, ICML}

\vskip 0.3in



\printAffiliationsAndNotice{}  



\section{Preliminary on Hyperplane Arrangements} 
In this section, we recall some basic knowledge on hyperplane arrangements \cite{zaslavsky1975facing,Stanley04anintroduction}, which will be used in the proofs of theorems in this paper.  
An affine hyperplane in a Euclidean space $V \simeq \mathbb{R}^n$ is a subspace with the following form:   $H = \{X \in V: \alpha \cdot X=b\}$,
where $``\cdot"$ denotes the inner product, $\0\neq \alpha\in V$ is called the {\em norm vector} of $H$, and $b\in \mathbb{R}$. For example, when $V = \mathbb{R}^n$, an affine hyperplane has the following form: $\{(x_1,x_2,\ldots,x_n)\in \mathbb{R}^n: \sum_{i=1}^n a_ix_i=b\}$ where $a_i,b\in\mathbb{R}$ and there exists some $i$ with $a_i\neq 0$.  A finite {\em hyperplane arrangement} $\mA$ of a Euclidean space $V$ is a finite set of affine hyperplanes in $V$. 
A {\em region} of an arrangement $\mA = \{H_i\subset V:1\leq i\leq m\}$ is defined as a connected component of $V\setminus (\cup_{i=1}^m H_i)$, which is a connected
component of the complement of the union of the hyperplanes in $\mA$. Let $r(\mA)$ denote the number of regions for an arrangement $\mA$.  It is natural to ask: What is the maximal number of regions for an arrangement with $m$ hyperplanes in $\mathbb{R}^n$? The following Zaslavsky's theorem answers this question. 

\begin{prop}[Zaslavsky's Theorem \cite{zaslavsky1975facing,Stanley04anintroduction}]\label{thm:ZaslavskyNN}
    Let $\mA = \{H_i\subset V:1\leq i\leq m\}$ be an arrangement in
    $\R^{n}$. Then, the number of regions for the arrangement $\mA$ satisfies
    \begin{eqnarray} \label{eq:region_general1}
        r(\mA)\leq\sum_{i=0}^{n} \binom{m}{i}.
\end{eqnarray} 
Furthermore, the above equality holds iff $\mA$ is in general position, i.e.,   (i) $\dim(\bigcap_{j=1}^k H_{i_j}) = n-k$ for any $k\leq n$ and $1\leq i_1<i_2<\cdots< i_j \leq m$; (ii)~~
$\bigcap_{j=1}^k H_{i_j}=\emptyset$~ for any $k> n$ and $1\leq i_1<i_2<\cdots< i_j \leq m$.
\end{prop}

For example, if $n=2$ then a set of lines is in general position if no two are parallel
and no three meet at a point. In this case, the number of regions of an arrangement $\mA$
with $m$ lines in general position is equal to 
\begin{equation}
    \label{eq:region_2d}
r(\mA) = 
\binom{m}{2} + m + 1. 
\end{equation}
For an arrangement $\mA$ and some $H_0\in \mA$, we define 
\[
\mA^{H_0}:= \{H\cap H_0: H\in \mA, H\neq H_0,~  H\cap H_0\neq \emptyset\}
\]
to be the set of nonempty intersections of $H_0$ and other hyperplanes in $\mA$. 
The following lemma gives a recursive method to compute $r(\mA)$. 
\begin{lem}[Lemma 2.1 from {\cite{Stanley04anintroduction}}]\label{lem:recursive}
 Let $\mA$ be an arrangement and  $H_0\in \mA$. Then we have 
\[
r(\mA) = r(\mA\setminus\{H_0\}) + r(\mA^{H_0}).
\]
\end{lem}
Lemma \ref{lem:recursive} means that we can calculate the number of regions of an arrangement by induction.

Let $\#\mA$ be the number of hyperplanes in $\mA$ and $\rank(\mA)$ be the dimension of the space spanned by the normal vectors of the hyperplanes in $\mA$. An arrangement $\mA$ is called {\em central} if $\bigcap_{H\in\mA}  H\neq\emptyset$. 

\begin{lem}[{Theorems 2.4 and 2.5 from \cite{Stanley04anintroduction}}]\label{lem:equality_hold}
Let $\mA$ be an arrangement in an $n$-dimensional vector space. Then we have  \[ r(\mA) = \sum_{\substack{\mB\subseteq\mA\\  \mB \text{ central}} } (-1)^{\#\mB-\rank(\mB)}. \] 
\end{lem}

\section{Proofs of Results for One-Layer CNNs}
Let $[n,m]:=\{n,n+1,n+2\ldots,m\}$ be the set of integers from $n$ to $m$ and $[m]:=[1,m]=\{1,2,\ldots,m\}$.
We establish the following generalization of Zaslavsky's theorem, which is crucial in the proof of Theorem $2$. 
\begin{prop}\label{thm:ZaslavskyCNN}
Let $V=\mathbb{R}^n$, $V_1,V_2,\ldots,V_m$ be $m$ nonempty subspaces of $V$, and $n_1,n_2,\ldots, n_m \in \mathbb{N}$ be some nonnegative integers. 
Let $\mA = \{H_{k,j}:1\leq k\leq m,~~ 1\leq j \leq n_k\}$ be an arrangement in $\R^{n}$ with $H_{k,j} = \{X \in V: \alpha_{k,j} \cdot X=b_{k,j}\}$ where  $\0\neq\alpha_{k,j} \in V_k,~ b_{k,j}\in\mathbb{R}$. Then, the number of regions for the arrangement $\mA$ satisfies 
\begin{eqnarray} \label{eq:region_general}
        r(\mA)\leq\sum_{(i_1,i_2,\ldots, i_m) \in K_{V;V_1,V_2,\ldots,V_m} }~~\prod_{k=1}^m \binom{n_k}{i_k},
\end{eqnarray} 
where 
\[
K_{V;V_1,V_2,\ldots,V_m}=\left\{(i_1,i_2,\ldots, i_m):i_k\in\mathbb{N}, ~~\sum_{k\in J} i_k \leq \dim\left(\sum_{k\in J} V_k\right) \forall J\subseteq [m]\right\}.
\]
Furthermore, assume that the following two conditions hold for the arrangement~$\mA$: 

(i) For each $(i_1,i_2,\ldots, i_m) \in K_{V;V_1,V_2,\ldots,V_m}$, any $\sum_{k=1}^m i_k$ vectors with $i_k$ distinct vectors chosen from the set $\{\alpha_{k,j}: 1\leq j \leq n_k\}$ are linear independent;

(ii) For each $(i_1,i_2,\ldots, i_m) \in \mathbb{N}^m\setminus K_{V;V_1,V_2,\ldots,V_m}$,  the intersection of any $\sum_{k=1}^m i_k$ hyperplanes with $i_k$ distinct hyperplanes chosen from the set $\{H_{k,j}: 1\leq j \leq n_k\}$ are empty.

Then, the equality in \eqref{eq:region_general} holds: \begin{eqnarray} \label{eq:region_general_equality}        r(\mA)=\sum_{(i_1,i_2,\ldots, i_m) \in K_{V;V_1,V_2,\ldots,V_m} }~~\prod_{k=1}^m \binom{n_k}{i_k}. \end{eqnarray}
\end{prop}
\begin{proof}
First, 
we will prove \eqref{eq:region_general} by induction on $\sum_{k=1}^m n_k$. When $\sum_{k=1}^m n_k=0$, both sides of \eqref{eq:region_general} equals $1$ since $\binom{0}{0}=1$.  When $\sum_{k=1}^m n_k=1$, both sides equals $2$ since $\binom{1}{0}+\binom{1}{1}=2$.
Suppose that the result is true for $\sum_{k=1}^m n_k\leq N$ for some $N \geq 1$. Now consider the case $\sum_{k=1}^m n_k = N+1$. Without loss of generality, assume $n_1\geq 1$.  Then $H_{1,1}\in\mA$. Notice that the translation $Y\rightarrow Y+Y_0$ for some $Y_0\in \mathbb{R}^n$ (i.e., translate all points in $\mathbb{R}$ by a vector $Y_0$) doesn't change the number of regions in $\mA$. Thus we can assume $b_{1,1}=0$. Then $H_{1,1}$ becomes an $(n-1)$-dimensional subspace of $V$. Replace $H_0$ in Lemma \ref{lem:recursive} with $H_{1,1}$, we obtain  
\begin{align}\label{eq:region_recursive}
    r(\mA) = r(\mA\setminus\{H_{1,1}\}) + r(\mA^{H_{1,1}}).
\end{align}
By induction hypothesis, we have
\begin{eqnarray} \label{eq:region_general3}
        r(\mA\setminus\{H_{1,1}\})\leq\sum_{(i_1,i_2,\ldots, i_m) \in K_{V;V_1,V_2,\ldots,V_m} }~~\binom{n_1-1}{i_1}\prod_{k=2}^m \binom{n_k}{i_k}
\end{eqnarray} 
and
\begin{eqnarray} \label{eq:region_general4}
        r(\mA^{H_{1,1}})\leq\sum_{(i_1,i_2,\ldots, i_m) \in K_{V\cap H_{1,1};V_1\cap H_{1,1},V_2\cap H_{1,1},\ldots,V_m\cap H_{1,1}} }~~\binom{n_1-1}{i_1}\prod_{k=2}^m \binom{n_k}{i_k}.
\end{eqnarray} 
Let's consider \eqref{eq:region_general4} first. Since $H_{1,1}$ is the orthogonal complement of the linear subspace generated by $\alpha_{1,1}$, and  $\0\neq \alpha_{1,1}\subset V_1$, we have $$H_{1,1}+V_1=V.$$
Let $V'_k=H_{1,1}\cap V_k$ for $1\leq k\leq m$.
Therefore, for each $J\subseteq[2,m]$, we have 
\begin{align}
&\dim\left(H_{1,1}\cap\left(V_1+\sum_{k\in J} V_{k}\right)\right) =\dim(H_{1,1})+\dim\left(V_1+\sum_{k\in J} V_{k}\right)-\dim(V)
= \dim\left(V_1+\sum_{k\in J} V_{k}\right)-1
\end{align}
and thus 
\begin{align}
&\dim\left(V'_1+\sum_{k\in J} V'_{k}\right) = \dim\left(V_1+\sum_{k\in J} V_{k}\right)-1.
\end{align}
On the other hand, it is trivial that \begin{align}
&\dim\left(\sum_{k\in J} V'_{k}\right) \leq \dim\left(\sum_{k\in J} V_{k}\right)
\end{align} 
for any $J\subseteq[2,m]$.
Therefore, by \eqref{eq:region_general4} we derive
\begin{align}
\label{eq:region_general5}
r(\mA^{H_{1,1}})
&\leq
\sum_{(i_1,i_2,\ldots, i_m) \in K_{H_{1,1};V'_1,V'_2,\ldots,V'_m} }~~\binom{n_1-1}{i_1}\prod_{k=2}^m \binom{n_k}{i_k} \nonumber
\\&\leq 
\sum_{ \substack{
i_1-1+\sum_{k\in J} i_k \leq \dim\left(V'_1+\sum_{k\in J} V'_k\right)~~ \forall J\subseteq [2,m] \\  \sum_{k\in J} i_k \leq \dim\left(\sum_{k\in J} V'_k\right)~~ \forall J\subseteq [2,m]
}
}~~\binom{n_1-1}{i_1-1}\prod_{k=2}^m \binom{n_k}{i_k}
\nonumber
\\&\leq 
\sum_{ \substack{
i_1+\sum_{k\in J} i_k \leq \dim\left(V_1+\sum_{k\in J} V_k\right)~~ \forall J\subseteq [2,m] \\  \sum_{k\in J} i_k \leq \dim\left(\sum_{k\in J} V_k\right)~~ \forall J\subseteq [2,m]
}
}~~\binom{n_1-1}{i_1-1}\prod_{k=2}^m \binom{n_k}{i_k}
\nonumber
\\&=
\sum_{(i_1,i_2,\ldots, i_m) \in K_{V;V_1,V_2,\ldots,V_m} }~~\binom{n_1-1}{i_1-1}\prod_{k=2}^m \binom{n_k}{i_k}.
\end{align}
Put \eqref{eq:region_recursive}, \eqref{eq:region_general3} and  \eqref{eq:region_general5} together, we obtain 

\begin{align} 
r(\mA)&\leq\sum_{(i_1,i_2,\ldots, i_m) \in K_{V;V_1,V_2,\ldots,V_m}  }~~\left(\binom{n_1-1}{i_1}\prod_{k=2}^m \binom{n_k}{i_k} + \binom{n_1-1}{i_1-1}\prod_{k=2}^m \binom{n_k}{i_k}\right) 
\nonumber
\\&=
\sum_{(i_1,i_2,\ldots, i_m) \in K_{V;V_1,V_2,\ldots,V_m}  }~~\prod_{k=1}^m \binom{n_k}{i_k},
\end{align} 
which competes the proof of \eqref{eq:region_general}. 

Furthermore, assume that the arrangement~$\mA$ satisfies the condition (i) and (ii). Then, the central sub-arrangements  of~$\mA$ are exactly the sub-arrangements $\mB$ consisting of $\sum_{k=1}^m i_k$ hyperplanes with $i_k$ distinct hyperplanes chosen from the set $\{H_{k,j}: 1\leq j \leq n_k\}$, where $(i_1,i_2,\ldots, i_m) \in K_{V;V_1,V_2,\ldots,V_m}$. In this case, $\#\mB=\rank(\mB)=\sum_{k=1}^m i_k$. Also, for any given $(i_1,i_2,\ldots, i_m) \in K_{V;V_1,V_2,\ldots,V_m}$, we have $\binom{n_k}{i_k}$ choices to pick $i_k$ hyperplanes from each $\{\alpha_{k,i}: 1\leq i \leq n_k\}$. Therefore, by Lemma \ref{lem:equality_hold} we obtain
\begin{align*} 
r(\mA) &= \sum_{\substack{\mB\subseteq\mA\\\mB \text{ central}} } (-1)^{\#\mB-\rank(\mB)} =\sum_{\substack{\mB\subseteq\mA\\\mB \text{ central}} } 1 = \sum_{(i_1,i_2,\ldots, i_m) \in K_{V;V_1,V_2,\ldots,V_m}  }~~\prod_{k=1}^m \binom{n_k}{i_k}. 
\end{align*}
\end{proof}

To prove Theorem $2$, we need the following lemmas on picking distinct elements from the union of certain sets.  
\begin{lem}\label{lem:set}
Let $S_1,S_2,\ldots S_m$ be $m$ finite sets, and $a_1,a_2,\ldots a_m$ be some nonnegative integers such that for any $I\subseteq [m]$,
\begin{align}\label{eq: ki_in_equlity}
\sum_{i\in I} a_i \leq \#\bigcup_{i\in I} S_i.    
\end{align}
Then, we can take $a_i$ elements from each $S_i$ such that these $\sum_{i=1}^m a_i$ elements are distinct. 
\end{lem}
\begin{proof}
We will prove this lemma by induction on $m$. When $m=1$, the claim is trivial. Now assume that the lemma holds for any $1\leq m<n$ and consider the case $m=n$.   
Without loss of generality,  we assume that  
there exists some $\emptyset \neq I\subseteq [n]$ such that (otherwise we can always increase some $a_i$ to make the following equality holds for some $I$)  
\begin{align}\label{eq: kiequlity}
\sum_{i\in I} a_i = \#\bigcup_{i\in I} S_i.    
\end{align}
The proof is divided into two cases. 
 
Case (1): 
There exists some $I$ satisfying \eqref{eq: kiequlity} with $\emptyset \neq I\neq [n]$. In this case, we can assume that $I=[r]$ for some $1\leq r\leq n-1$ by symmetry, i.e.,   
\begin{align}\label{eq: kiequlity_r}
\sum_{i=1}^r a_i = \#\bigcup_{i=1}^r S_i.    
\end{align}
Let
\[
S'_j =
S_{j+r}\setminus \bigcup_{i=1}^r S_i, \qquad 1\leq j\leq n-r. 
\]
Then $\left(\bigcup_{j\in J} S'_j\right) \cap \Bigl(\bigcup_{i=1}^r S_i \Bigr) = \emptyset$. Therefore, for any $J\subseteq [n-r]$, we have
\begin{align}\label{eq: kiequlity2}
\#\bigcup_{j\in J} S'_j =   \#\left(\bigcup_{j\in J} S'_j \cup \bigcup_{i=1}^r S_i \right) - \#\bigcup_{i=1}^r S_i =   \#\left(\bigcup_{j\in J} S_{j+r} \cup \bigcup_{i=1}^r S_i \right) - \#\bigcup_{i=1}^r S_i. 
\end{align}
By \eqref{eq: ki_in_equlity} and \eqref{eq: kiequlity_r} the above equality becomes 
\begin{align}\label{eq: kiequlity3}
    \#\bigcup_{j\in J} S'_j \geq    \left(\sum_{j\in J} a_{j+r} + \sum_{i=1}^r a_i  \right) - \sum_{i=1}^r a_i = \sum_{j\in J} a_{r+j}. 
\end{align}
Since $1\leq \# I \leq n -1$, by induction we can pick $a_i$ elements from each $S_i$ for $1\leq i \leq r$, and $a_{r+j}$ elements from each $S_{j+r}$ for $1\leq j \leq n-r$ such that these $\sum_{i=1}^n a_i$ elements are distinct. Thus the claim holds. 

Case (2):
The only $I$ satisfying \eqref{eq: kiequlity} is $I=[n]$. Then $\#S_1>a_1$ and thus $S_1 \cap \bigcup_{i=2}^n S_i \neq \emptyset$ (otherwise $\sum_{i=1}^n a_i=  \#\bigcup_{i=1}^n S_i=\#S_1+\#\bigcup_{i=2}^n S_i >\sum_{i=1}^n a_i$, a contradiction). Let $x\in S_1 \cap \bigcup_{i=2}^n S_i$ and  
\[
S'_j =
\begin{cases}
S_{j}, \qquad \qquad ~~ 2\leq j\leq n; \\
S_{j}\setminus \{x\}, \qquad j=1.
\end{cases}
\]
Then $\{S'_j:1\leq j\leq n\}$ still satisfies \eqref{eq: ki_in_equlity}. But $\sum_{i=1}^n \#S'_i < \sum_{i=1}^n \#S_i$. Then $\{S'_j:1\leq j\leq n\}$ either satisfies Case (1), which leads to a solution; or still in Case (2), which we can continue the process until Case (i) satisfies. This completes the proof.   
\end{proof}

\begin{lem}\label{lem:set2}
Let $S_1,S_2,\ldots S_m$ be $m$ finite sets. Then, there exist some $a_1,a_2,\ldots a_m \in \mathbb{N}$ such that
\begin{align}\label{eq: ki_in_equlity31}
\sum_{i=1}^m a_i = \#\bigcup_{i=1}^m S_i,   
\end{align}
and for any $I\subseteq [m]$,
\begin{align}\label{eq: ki_in_equlity21}
\sum_{i\in I} a_i \leq \#\bigcup_{i\in I} S_i.    
\end{align}
\end{lem}
\begin{proof}
We will prove it by Induction on $m$. The claim is trivial when $m=1$. Now assume that $m\geq 2$ and the result is true for $m-1$. Therefore, we can pick some $a_1,a_2,\ldots a_{m-1} \in \mathbb{N}$ such that
\begin{align}\label{eq: ki_in_equlity32}
\sum_{i=1}^{m-1} a_i = \#\bigcup_{i=1}^{m-1} S_i,    
\end{align}
and for any $I\subseteq [m-1]$,
\begin{align}\label{eq: ki_in_equlity22}
\sum_{i\in I} a_i \leq \#\bigcup_{i\in I} S_i.    
\end{align}
  
Furthermore, let $a_m=\#\left(S_m\setminus \bigcup_{i=1}^{m-1} S_i  \right)$. Then, for any $I\subseteq [m-1]$, we have
\begin{align}\label{eq: ki_in_equlity23}
a_m+\sum_{i\in I} a_i \leq \#\bigcup_{i\in I} S_i + \#\left(S_m\setminus \bigcup_{i=1}^{m-1}S_i  \right) \leq \#\bigcup_{i\in I\cup\{m\}} S_i.    
\end{align}
Also, 
\begin{align}\label{eq: ki_in_equlity33}
\sum_{i=1}^{m} a_i = \#\bigcup_{i=1}^{m-1} S_i + \#\left(S_m\setminus \bigcup_{i=1}^{m-1} S_i  \right) =\#\bigcup_{i=1}^{m} S_i.    
\end{align}  
Then the claim is also true for $m$.
\end{proof}

We also need the following lemmas on measure zero subsets of Euclidean spaces with respect to Lebesgue measure. 
\begin{lem}\label{lem:dependent_vector}
Let $V \cong \mathbb{R}^n$ be a vector space. Then
$S=\{(v_1,v_2,\ldots, v_n)\in V^n: v_1,v_2,\ldots, v_n~~\text{are linear dependent}\}$
is a measure zero subset of $V^n$, with respect to Lebesgue measure.
\end{lem}
\begin{proof}
Without loss of generality, assume $V = \mathbb{R}^n$. Let the $i$-th vector be $v_i=(x_{i,1},x_{i,2},\ldots,x_{i,n})$. Then $v_1,v_2,\ldots,v_n$ are linear dependent iff $$det((x_{i,j})_{n\times n}) = 0,$$ whose left hand side is a non-zero polynomial of all $x_{i,j}$. It is easy to see that the solution of this polynomial has co-dimension $1$ in $\mathbb{R}^{n\times n}$, thus $S$ is a measure zero set.  
\end{proof}

\begin{lem}\label{lem:condition2}
Let $m>n$ be two given positive integers, $A = (a_{ij})_{m\times n}\in \mathbb{R}^{m\times n}$ and $C=(c_1,c_2,\ldots,c_m) \in \mathbb{R}^{m}$. Let 
$S$ be the set of $(A,C)\in \mathbb{R}^{m(n+1)}$ such that 
\begin{equation*}\label{eq:3}
\begin{cases}
   \begin{aligned}
  & {a_{11}{x}_{1}}+{a_{12}{x}_{2}}+\cdots +{a_{1n}{x}_{n}}=c_1 \\
 & {a_{21}{x}_{1}}+{a_{22}{x}_{2}}+\cdots +{a_{2n}{x}_{n}}=c_2 \\
 & ~~~~~~~~~~~~~~~\vdots  \\
 & {a_{m1}{x}_{1}}+{a_{m2}{x}_{2}}+\cdots +{a_{mn}{x}_{n}}=c_m \\
\end{aligned}  
\end{cases} 
\end{equation*}
has solutions for $(x_1,x_2,\ldots,x_n)\in \mathbb{R}^n$. 
Then $S$ is a measure zero subset of $\mathbb{R}^{m(n+1)}$, with respect to Lebesgue measure.
\end{lem}
\begin{proof}
By Lemma \ref{lem:dependent_vector}, the augmented matrix 
$
(A,C)
$
has the rank $(n+1)$ except for a measure zero subset of $\mathbb{R}^{m(n+1)}$. On the other hand, the rank of the matrix $A$ is at most $n$. Therefore, the rank of the augmented matrix
$
(A,C)
$
is larger than the rank of $A$ except for a measure zero subset of $\mathbb{R}^{m(n+1)}$, thus by Rouch\'e-Capelli Theorem \cite{shafarevich2012linear} we obtain that \eqref{eq:3} has no solutions except for a measure zero set of $\mathbb{R}^{m(n+1)}$. 
\end{proof}

Lemma \ref{lem:set} implies the following results when we choose a basis of a linear space properly. 

\begin{lem}\label{cor:linear_space}
Let $V \cong \mathbb{R}^n$ be a vector space and $V_i\, (1\leq i\leq m)$ be $m$ subspaces of $V$. Suppose that some non-negative integers $a_i\,(1\leq i\leq m)$ satisfy 
$$
\sum_{i\in I} a_i \leq dim(\sum_{i\in I}{V_i}) 
$$
for each $I\subseteq [m]$. Then we obtain the following result.

(i) We can pick $a_i$ vectors from $V_i$ for $1\leq i\leq m$ such that these $\sum_{1\leq i\leq m} a_i$ vectors are linear independent. 

(ii) $\sum_{1\leq i\leq m} a_i$ vectors with $a_i$ vectors from $V_i$ for $1\leq i\leq m$ such that they are linear dependent, forms a measure zero set in $\prod_{i=1}^m V_i^{a_i}$, with respect to Lebesgue measure. 
\end{lem}
\begin{proof}
(i)
By linear algebra, we can construct a basis $v_1,v_2,\ldots,v_n$ of $V$ such that each $V_i$ has a basis which is a subset of $v_1,v_2,\ldots,v_n$. Then, by Lemma \ref{lem:set} this claim holds.

(ii)  Let $n'=\sum_{1\leq i\leq m} a_i$. By (i) there exist $n'$ linear independent vectors $v_1,v_2,\ldots,v_{n'}$ with $a_i$ vectors from $V_i$ for $1\leq i\leq m$. Let $V'_i$ be the vector spaces generated by such $a_i$ vectors in $V_i$. 
For any $n'$ linear dependent vectors $v'_1,v'_2,\ldots,v'_{n'}$ with $a_i$ vectors from $V_i$ for $1\leq i\leq m$, their projections $v''_1,v''_2,\ldots,v''_{n'}$ onto $\prod_{i=1}^m V'_i$ are also linear dependent.
Suppose that $v''_k =\sum_{j=1}^{n'} y_{k,j}v_j$ for $1\leq k\leq n'$. If $v'_k$ are chosen from $V_{i_1}$, such that $v_j \notin V_{i_1}  $, we set $y_{k,j} = 0$. Otherwise, we set $y_{k,j}=y'_{k,j}$.
Therefore, $\#\{y'_{k,j}\}$ equals the dimension of the projection of $ \prod_{i=1}^m V_i^{a_i} $ onto $\prod_{i=1}^m V'_i$.  
Also, $v''_1,v''_2,\ldots,v''_{n'}$ are linear dependent iff $$\det\left((y_{k,j})_{n'\times n'}\right) = 0.$$ 
Since $v_1,v_2,\ldots,v_{n'}$ are linear independent, the left hand side $\det\left((y_{k,j})_{n'\times n'}\right)$ must be a non-zero polynomial of some $y'_{k,j}$. Therefore, the solution of this polynomial forms a measure zero set in $\mathbb{R}^{\#\{y'_{k,j}\}}$ due to the zero measurability of the solutions of non-zero polynomial in Euclidean spaces (see \cite{lojasiewicz1964triangulation}). Thus such $\sum_{1\leq i\leq m} a_i$ vectors forms a measure zero set in $\prod_{i=1}^m V_i^{a_i}$, with respect to Lebesgue measure.
\end{proof}

Now we are ready to prove Theorem $2$. 
\begin{proof}[Proof of Theorem $2$]
By Definition $1$, the number of linear regions of $\mN$ at $\theta$ is equal to the number of regions of the hyperplane arrangement
\[
\mA_{\mN,\theta} := \{ H_{i,j,k}(X^0;\theta) :   1\leq i \leq n_1^{(1)}, ~ 1\leq j \leq n_1^{(2)}, 1\leq k\leq d_1 \},
\]
where $H_{i,j,k}(X^0;\theta)$ is the hyperplane determined by $Z^1_{i,j,k}(X^0;\theta)= 0$ (the expression of $Z^1_{i,j,k}(X^0;\theta)$ is given in (2)). Recall that $X^0=(X_{a,b,c}^{0})_{n_0^{(1)} \times n_0^{(2)} \times d_0}$. Then $H_{i,j,k}(X^0;\theta)$ can be written as 
$$
\langle \alpha_{i,j,k},X^0 \rangle_{F} +B^{1,k}= 0,   
$$
where $\langle \cdot,\cdot \rangle_{F}$ is the Frobenius inner product,  $\alpha_{i,j,k}$ is an $n_0^{(1)} \times n_0^{(2)} \times d_0$ dimensional tensor, whose $(a+(i-1)s_1,b+(j-1)s_1,c)$-th element is $W_{a,b,c}^{1,k}$ for all $1\leq a \leq f_1^{(1)} ,~  1\leq b \leq f_1^{(2)} ,~ 1\leq c \leq d_0$; and $0$ otherwise. Let 
$$
V_{i,j} = \{\beta \in \mathbb{R}^{n_0^{(1)} \times n_0^{(2)} \times d_0}:\beta_{a',b',c'} = 0 ~~\forall (a',b',c')\neq (a+(i-1)s_1,b+(j-1)s_1,c) \}
$$
be the subspace of $ \mathbb{R}^{n_0^{(1)} \times n_0^{(2)} \times d_0} $  
generated by $n_0^{(1)} \times n_0^{(2)} \times d_0$ dimensional tensors whose $(a+(i-1)s_1,b+(j-1)s_1,c)$-th element ranges over $\mathbb{R}$ for all $1\leq a \leq f_1^{(1)} ,~  1\leq b \leq f_1^{(2)} ,~ 1\leq c \leq d_0$; and $0$ otherwise. Then $\alpha_{i,j,k} \in V_{i,j} $ for $1\leq k\leq d_1$. By Proposition \ref{thm:ZaslavskyCNN}, we obtain
\begin{eqnarray} \label{eq:region_general10}
       R_{\mN,\theta}= r(\mA_{\mN,\theta})\leq\sum_{(t_{i,j})_{(i,j)\in I_\mN}  \in K_{V;(V_{i,j})_{(i,j)\in I_\mN} }}~~\prod_{k=1}^m \binom{d_1}{t_{i,j}},
\end{eqnarray} 
where 
\begin{align*}
& K_{V;(V_{i,j})_{(i,j)\in I_\mN}} 
=
\{(t_{i,j})_{(i,j)\in I_\mN}: \sum_{(i,j)\in J} t_{i,j} \leq \dim\left(\sum_{(i,j)\in J} V_{i,j}\right) \forall J\subseteq I_\mN\}
\\&=
\{(t_{i,j})_{(i,j)\in I_\mN}:t_{i,j}\in \mathbb{N},~~ \sum_{(i,j)\in J}t_{i,j}\leq \#\cup_{(i,j) \in J} S_{i,j}~~ \forall J \subseteq I_\mN \},
\end{align*}
which gives an upper bound for $R_{\mN,\theta}$ and $R_\mN$.
Next we will show that this upper bound can be reached except for a measure zero set in  $\mathbb{R}^{\#weights+\#bias}$ with respect to Lebesgue measure. By Lemmas \ref{lem:condition2} and \ref{cor:linear_space}, when $\theta$ ranges over $\mathbb{R}^{\#weights+\#bias}$, the set of $\theta$ such that $A_{\mN,\theta}$ satisfies the conditions (i) and (ii) of Proposition \ref{thm:ZaslavskyCNN} (replace $\{i_k:1\leq k\leq m\}$ by $\{t_{i,j}:(i,j)\in I_\mN\}$, and $\{V_{k}:1\leq k\leq m\}$ by $\{V_{i,j}:(i,j)\in I_\mN\}$), forms a complement of a measure zero set in $\mathbb{R}^{\#weights+\#bias}$,  with respect to Lebesgue measure. Then, for such parameters $\theta$, by Proposition \ref{thm:ZaslavskyCNN} we derive the equality holds for  \eqref{eq:region_general10}, which implies that the maximal number $R_\mN$ of linear regions of $\mN$ is equal to
\begin{align*}
R_\mN = \sum_{(t_{i,j})_{(i,j)\in I_\mN} \in K_\mN }~~\prod_{{(i,j)\in I}}\binom{d_1}{t_{i,j}},
\end{align*}
and the right hand side of the above equality also equals the expectation of the number $R_{\mN,\theta}$ of linear regions of $\mN$ with respect to the distribution $\mu$ of weights and biases. 
\end{proof}

The following result gives a simple example for Theorem $2$. 

\begin{cor}\label{th: asy_compare}
Let $\mN$ be a one-layer ReLU CNN with input dimension $1\times n\times 1$. Assume there are $d_1$ filters with dimension $1\times 2\times 1$ and stride $s=1$. Thus the hidden layer dimension is $1\times (n-1)\times d_1$. When $n$ is fixed, we have 
\begin{align}
R_{\mN} = \frac{(n-1)}{2}d_1^{n} + \O(d_1^{n-1}).
\end{align}
\end{cor}
\begin{proof}
By Theorem $2$,  we obtain
\begin{align}
R_{\mN} = \sum_{(t_{i,j})_{(i,j)\in I} \in K_{\mN} }~~\prod_{{(i,j)\in I}}\binom{d_1}{t_{i,j}}.
\end{align}
Furthermore, when $n$ is fixed, $R_{\mN}$ is a polynomial of $d_1$ with degree $n$ by Lemma $3$ in the main paper. 
To calculate the coefficient of the leading term $d_1^n$ of this polynomial, we need to determine all $ (t_{i,j})_{(i,j)\in I_\mN} \in K_{\mN}$ with $ \sum_{(i,j)\in I_\mN}t_{i,j}=n$. 
First, since $n_1^{(1)}=1$ and $n_1^{(2)}=n-1$,  it is easy to see that  $I_\mN=\{(1,j): ~ 1\leq j \leq n-1\}$ and $S_{1,j}=\{ (1,j,1),  (1,j+1,1)\}$ for each  $1\leq j \leq n-1$. Therefore, 
\begin{align}
  K_{\mN}=\{(t_{1,j})_{1\leq j\leq n-1}:~~ t_{1,j}\in \mathbb{N},~~
\sum_{j\in J}t_{1,j}\leq \#\cup_{(1,j) \in J} S_{1,j}~~ \forall J \subseteq [n-1] \}.
\end{align}
Then, there are $n-1$ vectors $ (t_{1,j})_{1\leq j\leq n-1} \in K_{\mN}$ satisfying $ \sum_{j=1}^{n-1}t_{1,j}=n$:
$(2,1,1,\ldots,1)$, $(1,2,1,\ldots,1)$, $(1,1,2,1,\ldots,1),~\ldots,~ (1,1,1,\ldots,1,2).$
Therefore, the leading term in $R_\mN$ equals
$$
(n-1)\binom{d_1}{2}d_1^{n-2} = \frac{(n-1)}{2}d_1^{n} - \frac{(n-1)}{2}d_1^{n-1}
$$
and thus 
\begin{align}
R_{\mN} = \frac{(n-1)}{2}d_1^{n} + \O(d_1^{n-1}).
\end{align}
This completes the proof. 
\end{proof}

Next, we prove Lemma $3$ and Theorem $3$ in the main paper.

\begin{proof}[Proof of Lemma $3$ in the main paper]
Directly replace $\{a_i:1\leq i\leq m\}$ by $\{t_{i,j}:(i,j)\in I_\mN\}$, and $\{S_{i}:1\leq i\leq m\}$ by $\{S_{i,j}:(i,j)\in I_\mN\}$ in Lemma $4$, we derive the result.
\end{proof}

\begin{proof}[Proof of Theorem $3$]
It is easy to see that $\binom{d_1}{t_{i,j}} = \Theta(d_1^{t_{i,j}})$ when $d_1$ tends to infinity. Then,
by Eq. (4) and Lemma $3$ in the main paper, we have
\begin{align}
    R_\mN  = \Theta(d_1^{ \#\cup_{(i,j) \in I_\mN} S_{i,j}  }).
\end{align}
Furthermore, if all input neurons have been involved in the convolutional calculation, we have 
\begin{align} 
\cup_{(i,j) \in I_\mN}  S_{i,j} = \{ (a,b,c): 1\leq a \leq n_0^{(1)} ,~ 
1\leq b \leq n_0^{(2)} ,~ 1\leq c \leq d_0 \} 
\end{align}
and thus
$$
R_\mN  = \Theta(d_1^{n_0^{(1)}\times n_0^{(2)}\times d_0}).
$$
\end{proof}

\section{Proofs of Results for Multi-Layer CNNs}
In this section, we prove Theorem $5$ on multi-layer ReLU CNNs.
\begin{proof}[Proof of Theorem $4$]
Assume that the parameters $W$ and $B$ for such two convolutional layers are the same as defined in Section $2$.
Let $l=1,2$ in (2) in the main paper and $X^l_{i,j,k}=Z^l_{i,j,k}(X^0;\theta)$, we obtain
\begin{align}\label{eq:twolayers1} 
X^1_{i,j,k}=\sum_{a=1}^{f_1^{(1)}}\sum_{b=1}^{f_1^{(2)}} \sum_{c=1}^{d_{0}}  W_{a,b,c}^{1,k} X^{0}_{a+(i-1)s_1,b+(j-1)s_1,c}   + B^{1,k}
\end{align}
and 
\begin{align}\label{eq:twolayers2} 
X^2_{i,j,k}=\sum_{a=1}^{f_2^{(1)}}\sum_{b=1}^{f_2^{(2)}} \sum_{c=1}^{d_{1}}  W_{a,b,c}^{2,k} X^{1}_{a+(i-1)s_2,b+(j-1)s_2,c}   + B^{2,k}.
\end{align}
Substitute \eqref{eq:twolayers1} into \eqref{eq:twolayers2}, we derive 
\begin{align}\label{eq:twolayers3} 
X^2_{i,j,k}&=\sum_{a'=1}^{f_2^{(1)}}\sum_{b'=1}^{f_2^{(2)}} \sum_{c'=1}^{d_{1}}   \sum_{a=1}^{f_1^{(1)}}\sum_{b=1}^{f_1^{(2)}} \sum_{c=1}^{d_{0}}  W_{a',b',c'}^{2,k}W_{a,b,c}^{1,c'} X^{0}_{a+(a'+(i-1)s_2-1)s_1,b+(b'+(j-1)s_2-1)s_1,c}   + const
\\&=
\sum_{a'=1}^{f_2^{(1)}}\sum_{b'=1}^{f_2^{(2)}} \sum_{c'=1}^{d_{1}}   \sum_{a=1}^{f_1^{(1)}}\sum_{b=1}^{f_1^{(2)}} \sum_{c=1}^{d_{0}}  W_{a',b',c'}^{2,k}W_{a,b,c}^{1,c'} X^{0}_{a+(a'-1)s_1+(i-1)s_1s_2,b+(b'-1)s_1+(j-1)s_1s_2,c}   + const.
\end{align}
Note that $1\leq a+(a'-1)s_1 \leq f_1^{(1)}+(f_2^{(1)}-1)s_1 $ and $1\leq b+(b'-1)s_1 \leq f_1^{(2)}+(f_2^{(2)}-1)s_1 $. Then \eqref{eq:twolayers3} becomes 
\begin{align}\label{eq:twolayers4} 
X^2_{i,j,k}=
\sum_{a=1}^{f_1^{(1)}+(f_2^{(1)}-1)s_1}
~
\sum_{b=1}^{f_1^{(2)}+(f_2^{(2)}-1)s_1} \sum_{c=1}^{d_{0}}   {W'}_{a,b,c}^{k} X^{0}_{a+(i-1)s_2,b+(j-1)s_2,c}  + const
\end{align}
where ${W'}_{a,b,c}^{k}$ are some constants. Therefore, $\mN$ is realized as a ReLU CNN with one hidden convolutional layer such that its $d_2$ filters has size $(f_1^{(1)}+(f_2^{(1)}-1)s_1)\times (f_1^{(2)}+(f_2^{(2)}-1)s_1)\times d_{0}  $ and stride $s_1s_2$, which completes the proof.  
\end{proof}

\begin{proof}[Proof of Theorem $5$]
(i) The basic idea is to map many regions of the input space of each layer to the same set, thus identify many regions of space.  

The $L=1$ case is guaranteed by Theorem $2$. 
Next, we consider the case $L\geq 2$. Let $p=\lfloor d_1/d_0 \rfloor$. We set    
\begin{align}\label{eq: ZX}
W_{a,b,c}^{1,k}=\begin{cases}
1,   \text{ if } a=b=1,  k=(c-1)p+1,~~ 1\leq c\leq d_0; \\
2,    \text{ if } a=b=1,~~  (c-1)p+2\leq k\leq cp, 1\leq c\leq d_0;\\
0,  \text{ otherwise }
\end{cases}
\end{align}
and 
\begin{align}\label{eq: ZX1}
B^{1,k}=
\begin{cases}
-(k-(c-1)p-1),   \text{ if } (c-1)p+1\leq k\leq cp \text{ for  some   } 1\leq c\leq d_0;
\\
0,  \text{ otherwise. }
\end{cases}
\end{align}
Therefore, by (2) in the main paper we obtain 
\begin{align}\label{eq: Z1X0}
Z^1_{i,j,k}(X^0;\theta)
=
\begin{cases}
 X^0_{1+(i-1)s_1,1+(j-1)s_1,c},   \text{ if } k=(c-1)p+1 \text{ for  some   }  1\leq c\leq d_0;\\
2X^0_{1+(i-1)s_1,1+(j-1)s_1,c} - (k-(c-1)p-1),   \text{ if } (c-1)p+2\leq  k  \leq cp \text{ for  some   } 1\leq c\leq d_0;
\\
0,  \text{ otherwise. }
\end{cases}
\end{align}
When $ W_{a,b,c}^{1,k}$ and $B^{1,k}$ are given as in \eqref{eq: ZX} and \eqref{eq: ZX1},  the map 
\begin{align}\label{X0toX1*}
X^1_{i,j,k} =   \max\{0,Z^1_{i,j,k}(X^0;\theta)\} 
\end{align}
determines a function
\begin{align}\label{X0toX1}
    X^1 = \Phi_1 (X^0) 
\end{align}
from $\mathbb{R}^{n_0^{(1)} \times n_0^{(2)} \times d_0} $ to $  \mathbb{R}^{n_1^{(1)} \times n_1^{(2)} \times d_1}$.

For each $i,j\in \mathbb{N}^+$, let 
\begin{align}\label{eq: psix}
\psi_{i}(x)= 
\begin{cases}
\max\{ 0, x\} ,   \text{ if } i=1;\\
\max\{ 0, 2x - (i-1) \},   \text{ if } i\geq 2
\end{cases}
\end{align}
and 
\begin{align}
\phi_{j}(x)= \sum_{i=1}^j (-1)^{i+1} \psi_{i}(jx). 
\end{align}

Then it is easy to check that 

\begin{align}
\phi_{j}(x)= \begin{cases}
0,   \text{ if } x\leq 0;\\
jx - i,   \text{ if } \frac{i}{j} \leq x \leq \frac{2i+1}{2j}\leq \frac{1}{2} \text{ where } i\in\mathbb{N};\\
i-jx,   \text{ if } \frac{2i-1}{2j} \leq x \leq \frac{i}{j} \leq \frac{1}{2} \text{ where } i\in\mathbb{N}^+, 
\end{cases} 
\end{align}
which means that $\phi_{j}$ is an affine function when restricted to each interval $[0,\frac{1}{2j} ],[\frac{1}{2j},\frac{2}{2j}],\ldots, [\frac{j-1}{2j},\frac{1}{2}]$
and furthermore
$\phi_{j}([0,\frac{1}{2j} ])=\phi_{j}( [\frac{1}{2j},\frac{2}{2j}])=\cdots =\phi_{j}(  [\frac{j-1}{2j},\frac{j}{2j}]) =  [0,\frac{1}{2}]$ (i.e., $\phi_{j}(x)$ sends $j$ distinct intervals $[0,\frac{1}{2j} ],[\frac{1}{2j},\frac{2}{2j}],\ldots, [\frac{j-1}{2j},\frac{1}{2}]$  to the same interval $[0,\frac{1}{2}]$).

Next, we define an intermediate convolutional layer (without activation functions) from 
\begin{align*}
X^1=(X^1_{a,b,c})_{n_1^{(1)} \times n_1^{(2)}\times d_1}
\end{align*}
to
\begin{align*}
Y^1=(Y^1_{a,b,c})_{n_1^{(1)} \times n_1^{(2)}\times d_0}
\end{align*}
between the first and second hidden convolutional layers.
We set the $d_0$ filters with size $1\times 1 \times d_1$, the stride $1$, and define the weights $W'$ and biases $B'$ in this intermediate convolutional layer as 
\begin{align}\label{eq: ZX2}
{W'}_{1,1,k}^{1,c}=\begin{cases}
p\cdot (-1)^{i+1},   \text{ if }  k=(c-1)p+i,~~ 1\leq c\leq d_0; \\
0,  \text{ otherwise }
\end{cases}
\end{align}
and 
\begin{align}\label{eq: ZX3}
{B'}^{1,k}= 0 ~~\forall~~ 1\leq k \leq d_0.
\end{align}

Then by (2) in the main paper,  
\begin{align} \label{eq: YX}
Y^1_{a,b,c} &= 
p \sum_{i=1}^p (-1)^{i+1}   X^1_{a,b,(c-1)p+i}  
\end{align}
for $1\leq a \leq n_1^{(1)}$, $1\leq b \leq n_1^{(2)}$, $1\leq c \leq d_0$.
Therefore, \eqref{eq: YX} determines an affine function
\begin{align}\label{X0toX1**}
    Y^1 = \Phi'_1 (X^1) 
\end{align}
from  $  \mathbb{R}^{n_1^{(1)} \times n_1^{(2)} \times d_1}$ to $\mathbb{R}^{n_1^{(1)} \times n_1^{(2)} \times d_0} $.
Therefore, we obtain 
\begin{align} \label{eq: YX*}
Y^1_{a,b,c} &= 
p \sum_{i=1}^p (-1)^{i+1}   X^1_{a,b,(c-1)p+i} \nonumber
\\&=
p \sum_{i=1}^p (-1)^{i+1}   \max\{0,Z^1_{a,b,(c-1)p+i}\}  \nonumber
\\&=
\sum_{i=1}^p (-1)^{i+1}   \psi_{i}(pX^0_{1+(a-1)s_1,1+(b-1)s_1,c})  \nonumber
\\&= 
\phi_{p}(X^0_{1+(a-1)s_1,1+(b-1)s_1,c}).
\end{align}
The third equality holds due to Eqs. \eqref{eq: Z1X0} and \eqref{eq: psix}. 
By the previous discussion on properties of the function $\phi_{j}(x)$, the following map $\Psi_1 = \Phi'_1 \circ\Phi_1$ determined by Eq. \eqref{eq: YX*} 
\begin{align*}
\Psi_1 : \mathbb{R}^{n_0^{(1)} \times n_0^{(2)} \times d_0} &\stackrel{\Phi_1}{\longrightarrow}  \mathbb{R}^{n_1^{(1)} \times n_1^{(2)} \times d_1} \stackrel{\Phi'_1}{\longrightarrow} \mathbb{R}^{n_1^{(1)} \times n_1^{(2)}\times d_0} 
\\ X^0~~~~ &~~ \mapsto ~~~~~~~~~~ X^1 ~~~~~~~~~~ \mapsto ~~~~ Y^1
\end{align*}
sends $\lfloor\frac{d_1}{d_0}\rfloor^{n_1^{(1)} \times n_1^{(2)}\times d_0} = p^{n_1^{(1)} \times n_1^{(2)}\times d_0}$ distinct hypercubes 
$$
\left\{[0,\frac{1}{2p}], [\frac{1}{2p},\frac{2}{2p}], \cdots, [\frac{p-1}{2p},\frac{p}{2p}]\right\}^{{n_0^{(1)} \times n_0^{(2)}\times d_0}}
$$
in 
$[0,\frac{1}{2}]^{{n_0^{(1)} \times n_0^{(2)}\times d_0}}$ onto the same hypercube $[0,\frac12]^{n_1^{(1)} \times n_1^{(2)}\times d_0}$ of the intermediate layer $Y^1\in \mathbb{R}^{n_1^{(1)} \times n_1^{(2)}\times d_0} $ (this map is affine and bijective when restricted to each of the $\left\lfloor\frac{d_1}{d_0}\right\rfloor^{n_1^{(1)} \times n_1^{(2)}\times d_0}$ distinct hypercubes). 
Similarly (keep $d_0$ unchanged, and replace $n_0^{(1)}, n_0^{(2)}, n_1^{(1)} , n_1^{(2)}, d_1$ in $\Psi_1$ by $n_{l-1}^{(1)}, n_{l-1}^{(2)}, n_l^{(1)} , n_l^{(2)}, d_l$), we can define $\Phi_l,\Phi'_l,\Psi_l$ and $Y^l$ for $2\leq l\leq L-1$ such that the map
\begin{align*}
\Psi_{l} : \mathbb{R}^{n_{l-1}^{(1)} \times n_{l-1}^{(2)} \times d_0} &\stackrel{\Phi_l}{\longrightarrow}  \mathbb{R}^{n_l^{(1)} \times n_l^{(2)} \times d_l} \stackrel{\Phi'_l}{\longrightarrow} \mathbb{R}^{n_l^{(1)} \times n_l^{(2)}\times d_0} 
\\ Y^{l-1}~~~~ & ~ \mapsto ~~~~~~~~~ X^{l-1} ~~~~~~~~ \mapsto ~~~~ Y^l
\end{align*}
sends $\lfloor \frac{d_l}{d_0} \rfloor^{n_l^{(1)} \times n_l^{(2)}\times d_0}$ distinct hypercubes
$$
\left\{[0,\frac{1}{2p}], [\frac{1}{2p},\frac{2}{2p}], \cdots, [\frac{p-1}{2p},\frac{p}{2p}]\right\}^{{n_{l-1}^{(1)} \times n_{l-1}^{(2)}\times d_0}}
$$
in 
$[0,\frac{1}{2}]^{{n_{l-1}^{(1)} \times n_{l-1}^{(2)}\times d_0}}$ onto the hypercube $[0,\frac12]^{n_l^{(1)} \times n_l^{(2)}\times d_0}$ of the intermediate layer  $Y^l\in \mathbb{R}^{n_l^{(1)} \times n_l^{(2)}\times d_0} $. 
Therefore, 
\begin{align*}
\Psi_{L-1} \circ \Psi_{L-2} \circ \cdots  \circ \Psi_{2} \circ \Psi_{1}  : \mathbb{R}^{n_0^{(1)} \times n_0^{(2)} \times d_0}  &\rightarrow \mathbb{R}^{n_{L-1}^{(1)} \times n_{L-1}^{(2)}  \times d_0 }  
\\ X^0 ~~~~ &\mapsto ~~~~ {Y}^{L-1}
\end{align*}
sends $\prod_{l=1}^{L-1}\left\lfloor\frac{d_l}{d_0}\right\rfloor^{n_l^{(1)} \times n_l^{(2)}\times d_0}$ distinct hypercubes in $[0,\frac{1}{2}]^{{n_0^{(1)} \times n_0^{(2)}\times d_0}}$ onto the same hypercube $[0,\frac12]^{n_{L-1}^{(1)} \times n_{L-1}^{(2)}\times d_0}$ of the intermediate layer. Note that $\Phi_l  \circ \Phi'_{l-1}$ is the convolutional layer between $X^{l-1}$ and $X^l$ which has $d_l$ filter with size $f_l^{(1)}\times f_l^{(2)}\times d_{l-1}$ and stride $s_l$ due to Theorem $4$.   
Finally, by Theorem $2$, a one-layer ReLU CNN with input dimension $n_{L-1}^{(1)} \times n_{L-1}^{(2)} \times d_0$ and output dimension $n_L^{(1)} \times n_L^{(2)} \times d_L$ can divide the hypercube $[0,\frac12]^{n_{L-1}^{(1)} \times n_{L-1}^{(2)}\times d_0}$ into $R_{\mN'}$ regions. Put the network from $ X^{0} $ to $ {Y}^{{L-1}}$ and $ {Y}^{{L-1}}$ to $ {X}^{{L}}$ together, we prove the lower bound claim.  

(ii) We will prove this claim by induction on $L$.   
When $L=1$, by Theorem $2$ the claim is true. Now suppose that $L\geq 2$ and the claim is true for $L-1$. Let $\mN^*$ be the CNN obtained from $\mN$ by deleting the $L$-th hidden layer (i.e., $\mN^*$ consists of the first to the $L-1$-th layer of $\mN$). Then by induction hypothesis, we have
$$
R_{\mN^*} \leq R_{\mN''}\prod_{l=2}^{L-1}\sum_{i=0}^{n_0^{(1)} n_0^{(2)} d_0} \binom{n_l^{(1)}  n_l^{(2)}   d_l}{i}.
$$
Now we consider the $L$-th layer. Suppose that the CNN $\mN^*$ with parameters $\theta$ partitions the input space into $m$ distinct linear regions  $\mR_i ~ (1\leq i \leq m)$. Since each linear region $\mR_i$ corresponds to a certain activation pattern, the function   $\mathcal{F}_{\mN',\theta}$  
becomes an affine function when restricted to $\mR_i$. Therefore, after adding the $L$-th layer to $\mN^*$, when restricted to $\mR_i$, the function $\mathcal{F}_{\mN,\theta}\mid_{\mR_i}$ can be realised as a one-layer NN with  $n_0^{(1)} n_0^{(2)} d_0$ input neurons and $ n_l^{(1)}  n_l^{(2)}   d_l$ hidden neurons. By Proposition \ref{thm:ZaslavskyNN}, ${\mN}$ partitions $\mR_i$ into $\sum_{i=0}^{n_0^{(1)} n_0^{(2)} d_0} \binom{n_L^{(1)}  n_L^{(2)}   d_L}{i}$ distinct linear regions. Finally, we obtain
$$
R_\mN \leq R_{\mN^*} \sum_{i=0}^{n_0^{(1)} n_0^{(2)} d_0} \binom{n_L^{(1)}  n_L^{(2)}   d_L}{i} \leq
 R_{\mN''}\prod_{l=2}^{L}\sum_{i=0}^{n_0^{(1)} n_0^{(2)} d_0} \binom{n_l^{(1)}  n_l^{(2)}   d_l}{i},
$$
which completes the proof.
\end{proof}

\section{Calculation of the Number of Parameters for CNNs}

\begin{proof}[Proof of Lemma $4$ in the main paper]
For the $l$-th layer, the $k$-th weight matrix $W^{l,k}$ has $f_l^{(1)}\times f_l^{(2)}\times d_{l-1}$ entries and there are $d_l$ such weight matrices. The bias vector has length $d_l$. 
Thus there are $f_l^{(1)}\times f_l^{(2)}\times d_{l-1}\times d_l + d_l$ parameters in the $l$-th hidden layer.  Let $l$ range from $1$ to $L$, the  total number  of parameters equals 
$
 \sum_{l=1}^L \left(f_l^{(1)}\times f_l^{(2)}\times d_{l-1}\times d_l + d_l\right).
$
\end{proof}

\section{More Examples on the Maximal Number of Linear Regions for  One-Layer ReLU CNNs}
In this section, we list more examples on maximal number of linear regions for one-layer ReLU CNNs from Tables $1$~to~$5$, which is calculated according to Theorem $2$ in the main paper.

\begin{table}
\caption{The results for the maximal number of linear regions for a one-layer ReLU CNN with input dimension $2\times 2\times 1$,  $d_1$ filters with dimension $1\times 2\times 1$, stride $s=1$, and hidden layer dimension $2\times 1\times d_1$.}
\begin{center}
\resizebox{0.9\linewidth}{!}{
\begin{tabular}{c|c|c|c|c|c|c|c|c}
\hline
&$d_1=1$&$d_1=2$&$d_1=3$&$d_1=4$&$d_1=5$&$d_1=6$&$d_1=7$&$d_1=8$
\\
\hline
$R_\mN$ by Theorem $2$ & 4&16&49&121&256&484&841&1369
\\
\hline
Upper bounds by Theorem $1$ & 4&16&57&163&386&794&1471&2517
\\
\hline
Naive upper bounds & 4&16&64&256&1024&4096&16384&65536
\\
\hline
\end{tabular}\label{tab:1}
}
\end{center}
\end{table}

\begin{table}
\caption{The results for the maximal number of linear regions for a one-layer ReLU CNN with input dimension $1\times 4\times 1$,  $d_1$ filters with dimension $1\times 2\times 1$, stride $s=1$, and hidden layer dimension $1\times 3\times d_1$.}
\begin{center}
\resizebox{0.9\linewidth}{!}{
\begin{tabular}{c|c|c|c|c|c|c|c|c}
\hline
&$d_1=1$&$d_1=2$&$d_1=3$&$d_1=4$&$d_1=5$&$d_1=6$&$d_1=7$&$d_1=8$
\\
\hline
$R_\mN$ by Theorem $2$ & 8&55&217&611&1396&2773&4985&8317
\\
\hline
Upper bounds by Theorem $1$ & 8&57&256&794&1941&4048&7547&12951
\\
\hline
Naive upper bounds & 8&64&512&4096&32768&262144&2097152&16777216
\\
\hline
\end{tabular}\label{tab:2}
}
\end{center}
\end{table}

\begin{table}
\caption{The results for the maximal number of linear regions for a one-layer ReLU CNN with input dimension $2\times 3\times 1$,  $d_1$ filters with dimension $2\times 2\times 1$, stride $s=1$, and hidden layer dimension $2\times 1\times d_1$.}
\begin{center}
\resizebox{0.9\linewidth}{!}{
\begin{tabular}{c|c|c|c|c|c|c|c|c}
\hline
&$d_1=1$&$d_1=2$&$d_1=3$&$d_1=4$&$d_1=5$&$d_1=6$&$d_1=7$&$d_1=8$
\\
\hline
$R_\mN$ by Theorem $2$ & 4&16&64&247&836&2424&6126&13829
\\
\hline
Upper bounds by Theorem $1$ & 4&16&64&247&848&2510&6476&14893
\\
\hline
Naive upper bounds & 4&16&64&256&1024&4096&16384&65536
\\
\hline
\end{tabular}\label{tab:3}
}
\end{center}
\end{table}

\begin{table}
\caption{The results for the maximal number of linear regions for a one-layer ReLU CNN with input dimension $6\times 6\times 1$,  $d_1$ filters with dimension $1\times 3\times 1$, stride $s=2$, and hidden layer dimension $3\times 2\times d_1$.}
\begin{center}
\resizebox{0.9\linewidth}{!}{
\begin{tabular}{c|c|c|c|c|c|c|c|c}
\hline
&$d_1=1$&$d_1=2$&$d_1=3$&$d_1=4$&$d_1=5$&$d_1=6$&$d_1=7$&$d_1=8$
\\
\hline
$R_\mN$ by Theorem $2$ & 64 & 4096 & 250047 & 9129329 & 191102976 & 2537716544 & 23664622311 & 167557540697  
\\
\hline
Upper bounds by Theorem $1$ & 64 & 4096 & 262144 & 16777216 & 1073741824 & 68719476736 & 4398045536122 & 281443698512817 
\\
\hline
Naive upper bounds & 64 & 4096 & 262144 & 16777216 & 1073741824 & 68719476736 & 4398046511104 & 281474976710656
\\
\hline
\end{tabular}\label{tab:4}
}
\end{center}
\end{table}

\begin{table}
\caption{The results for the maximal number of linear regions for a one-layer ReLU CNN with input dimension $3\times 3\times 2$,  $d_1$ filters with dimension $2\times 2\times 2$, stride $s=1$, and hidden layer dimension $2\times 2\times d_1$.}
\begin{center}
\resizebox{0.9\linewidth}{!}{
\begin{tabular}{c|c|c|c|c|c|c|c|c}
\hline
&$d_1=1$&$d_1=2$&$d_1=3$&$d_1=4$&$d_1=5$&$d_1=6$&$d_1=7$&$d_1=8$
\\
\hline
$R_\mN$ by Theorem $2$ & 16 & 256 & 4096 & 65536 & 1048555 & 16721253 & 256376253 & 3459170397 
\\
\hline
Upper bounds by Theorem $1$ & 16 & 256 & 4096 & 65536 & 1048555 & 16721761 & 256737233 & 3485182163 
\\
\hline
Naive upper bounds & 16 & 256 & 4096 & 65536 & 1048576 & 16777216 & 268435456 & 4294967296 
\\
\hline
\end{tabular}\label{tab:5}
}
\end{center}
\end{table}

\newpage

\bibliography{number_region_CNN}
\bibliographystyle{icml2020}